\documentclass{article} 
\usepackage{iclr2027_conference,times}

\usepackage{amsmath,amsfonts,bm}

\def\eqref#1{equation~\ref{#1}}

\def\1{\bm{1}}

\DeclareMathAlphabet{\mathsfit}{\encodingdefault}{\sfdefault}{m}{sl}
\SetMathAlphabet{\mathsfit}{bold}{\encodingdefault}{\sfdefault}{bx}{n}

\DeclareMathOperator*{\argmax}{arg\,max}
\DeclareMathOperator*{\argmin}{arg\,min}

\usepackage{hyperref}
\usepackage{cleveref}
\usepackage{url}
\usepackage{graphicx}
\usepackage{caption}
\usepackage{subcaption}
\usepackage{kotex}
\usepackage{amssymb}
\usepackage{mathtools}
\usepackage{float}
\usepackage{wrapfig}
\usepackage{amsmath}
\usepackage{amsthm}
\usepackage{xcolor} 
\usepackage{cancel}
\usepackage{algorithm}
\usepackage{algpseudocode}
\usepackage[dvipsnames]{xcolor}
\usepackage{titletoc}
\usepackage[inkscapelatex=false]{svg}

\definecolor{ourpurple}{HTML}{CB297B}
\definecolor{ourblue}{HTML}{5a82ff}
\definecolor{ourmiddle}{HTML}{66509B}
\definecolor{ourgray}{HTML}{708090}
\definecolor{SphereBlue}{HTML}{87A4DF}

\hypersetup{
  colorlinks=true,
  linkcolor=SphereBlue,
  citecolor=SphereBlue,
  urlcolor=black
}

\Crefname{equation}{Eq.}{Eqs.}
\crefname{equation}{Eq.}{Eqs.}

\Crefname{figure}{Figure}{Figures}
\crefname{figure}{Figure}{Figures}

\newtheorem{assumption}{Assumption}
\Crefname{assumption}{Assumption}{Assumptions}
\crefname{assumption}{assumption}{assumptions}

\Crefname{proposition}{Proposition}{Propositions}
\crefname{proposition}{proposition}{propositions}

\newtheorem{lemma}{Lemma}
\Crefname{lemma}{Lemma}{Lemmas}
\crefname{lemma}{lemma}{lemmas}

\title{Hyperspherical Contrastive Representations\\Unlock Planning and Exploration}

\author{
\textbf{Jonghae Park}\textsuperscript{1}\thanks{Equal contribution.}\;\;
\textbf{Sungil Son}\textsuperscript{1}\footnotemark[1]\;\;
\textbf{Jonghun Shin}\textsuperscript{1}\;\;
\textbf{Junseo Lee}\textsuperscript{1}
\\
\,\!
\textbf{H. Jin Kim}\textsuperscript{1}\;\;
\textbf{Benjamin Eysenbach}\textsuperscript{2}
\\[0.4em]
\textsuperscript{1}Seoul National University\;
\textsuperscript{2}Princeton University
\\[0.2em]
\!
\texttt{\{bdfire1234,aldridge99\}@snu.ac.kr}
}

\iclrfinalcopy 
\begin{document}

\clearpage
\maketitle
\vspace{-1em}
\begin{abstract}
We develop a self-supervised reinforcement learning algorithm that learns to build towers of up to eight cubes from scratch, without prior data, rewards, demonstrations, or pretrained models.
Our method, \textbf{Contrastive Planning via Spherical Linear Interpolation (C-Slerp)}, uses contrastive representations as reusable \textit{building blocks} and repeatedly composes them to infer a sequence of intermediate waypoints.
Under our hyperspherical parameterization, the learned representations allow intermediate waypoints to be inferred in closed form through spherical linear interpolation, without training a separate high-level policy or performing graph search.
Importantly, the representation geometry also tells us how to adjust the number of waypoints based on the difficulty of each goal-reaching problem.
Experiments show that our method succeeds on long-horizon navigation and robotic manipulation tasks, including settings where strong baselines fail to achieve a single success.
Ablations further show that using a fixed number of waypoints substantially degrades performance, highlighting the importance of adapting the plan to task difficulty.
\begin{center}
Project page: 
\texttt{\href{https://jonghaepark.github.io/c-slerp}{\textcolor{magenta}{https://jonghaepark.github.io/c-slerp}}}
\end{center}
\end{abstract}

\begin{figure}[h]
    \centering
    \vspace{-0.0pt}
    \includegraphics[width=0.975\linewidth]{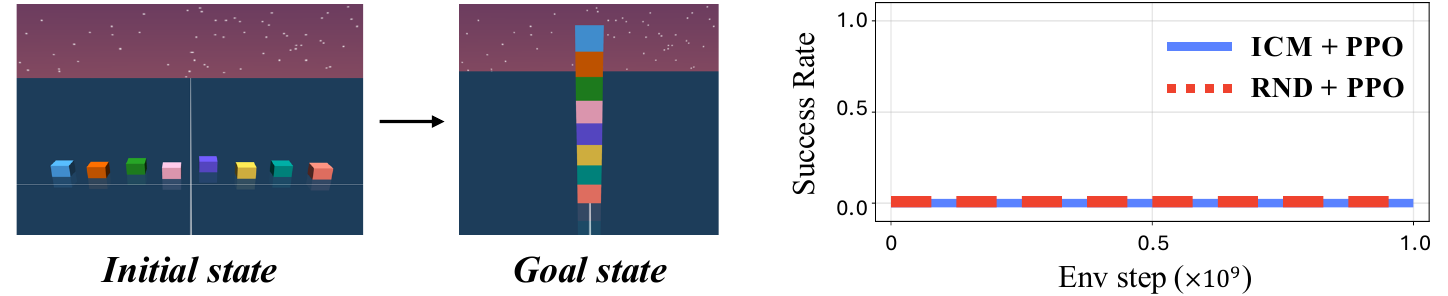}
    \vspace{-5.0pt}
    \caption{\textbf{Our paper studies long-horizon, sparse-reward tasks, including eight-cube stacking.} Given only the current and goal states, the agent aims to manipulate the eight cubes into the goal state without receiving any external reward along the way. Strong exploration methods~\citep{pathak2017curiosity,burda2018exploration} fail to discover a \textit{single} successful state in one billion environment steps.}
    \vspace{-5.0pt}
    \label{fig:fig1}
\end{figure}

\section{Introduction}
Consider a robot asked either to place a single cube or to build an eight-cube tower. The former can be solved directly, whereas the latter may benefit from a \textit{plan} with several intermediate stages.
Too many stages needlessly complicate the single-cube task, while too few would leave the tower-building task too difficult to solve.
The number of planning stages should therefore match the difficulty of the problem.

In goal-conditioned reinforcement learning~\citep{kaelbling1993learning}, these \textit{stages} can be defined by subgoals, or \textit{waypoints}, that divide long-horizon tasks into shorter subproblems. 
However, existing hierarchical approaches often use a fixed number of waypoints~\citep{nasiriany2019planning,zhang2021c} or generate waypoints that are a fixed number of steps into the future~\citep{nachum2018data,park2023hiql}, rather than explicitly selecting how many stages to plan for each problem.

A natural way to plan over arbitrary distances is to define a reusable component and compose it as many times as needed.
Contrastive reinforcement learning (CRL)~\citep{eysenbach2022contrastive} offers a promising foundation for this approach. CRL learns contrastive representations that encode temporal relationships between states under a single geometric distribution, providing natural \textit{building blocks} for multi-stage planning.
While these representations are primarily used to train goal-conditioned policies, we ask whether they can also be leveraged to infer an adaptive number of waypoints.

To this end, we introduce \textbf{Contrastive Planning via Spherical Linear Interpolation (C-Slerp)}, a hierarchical approach where a high-level planner provides waypoints for a low-level policy to follow.
Our main idea is to make this hierarchy \textit{adaptive} by using contrastive representations as building blocks for each stage of a plan.
By repeatedly composing these building blocks, we can score how well the state and goal are connected through different numbers of stages.
Using these scores, we select an appropriate number of stages and infer the corresponding intermediate waypoints.

A key advantage of our contrastive representations is that their spherical geometry enables intermediate waypoints to be inferred in closed form without training a separate high-level policy or performing graph search. Specifically, our high-level policy introduces \textit{no} additional network parameters --- it simply performs spherical linear interpolation in the learned representation space.

In this paper, we show how choosing the correct amount of planning can unlock new capabilities for reinforcement learning. We evaluate our approach on long-horizon navigation and robotic manipulation tasks, including settings in which strong exploration methods never achieve a single success in one billion environment steps (\Cref{fig:fig1}, Appendix~\ref{sec:experimental_details}).
Notably, our method successfully solves challenging benchmark tasks without a goal curriculum or explicit dense rewards, even where all baselines fail completely.
Our ablations further show that hierarchy alone is insufficient: fixing the number of planning stages substantially degrades performance.
Together, these results suggest that a key benefit of hierarchy for goal-reaching may lie in selecting an appropriate amount of planning for each problem rather than in making the policy more expressive.

\section{Related Work}
\paragraph{Hierarchical goal-conditioned reinforcement learning.}
Goal-conditioned reinforcement learning (GCRL) trains a single policy to reach diverse goals~\citep{schaul2015universal,andrychowicz2017hindsight}.
Despite recent progress, reaching temporally distant goals remains a challenging problem.
Hierarchical approaches~\citep{vezhnevets2017feudal, levy2017learning} address this challenge by decomposing long-horizon tasks into shorter subproblems.
Existing methods learn high-level policies that generate subgoals for low-level controllers~\citep{nachum2018data,zhang2020generating,park2023hiql} or infer waypoints to guide policy learning or exploration~\citep{chane-sane2021ris,zhang2021c}.
However, many of these approaches use a fixed number of waypoints or a fixed temporal horizon for subgoals, rather than explicitly selecting how many waypoints to plan for each state-goal pair.
Graph-based methods offer an alternative by searching over observed states or learned landmarks to construct waypoint sequences~\citep{eysenbach2019search,zhang2021l3p}.
Although these methods generate a \textit{variable} number of waypoints, they require additional computation for graph construction and depend on the coverage of states included in the graph.

\paragraph{Understanding the benefits of hierarchy.}
Prior work shows that hierarchy can improve exploration through temporally extended behavior~\citep{nachum2019does} and improve statistical and computational efficiency by reusing solutions to repeated subproblems~\citep{wen2020efficiency}.
For policy learning, HIQL~\citep{park2023hiql} shows that hierarchy provides clearer policy-learning signals, reducing sensitivity to value estimation errors.
Further, SAW~\citep{zhou2026flattening} brings the benefits of hierarchy to \textit{flat} policies through bootstrapping, without a separate subgoal generator.
These findings suggest that the benefits of hierarchy are closely tied to how it structures exploration and policy learning, rather than simply to the expressivity gained from a hierarchical policy structure.

\paragraph{Contrastive representations for goal-reaching.}
Contrastive reinforcement learning (CRL)~\citep{eysenbach2022contrastive} learns representations for goal-reaching in a self-supervised manner, without externally specified rewards.
It formulates goal-conditioned RL as a contrastive learning problem, jointly learning policies and representations that encode temporal relationships between states.
Subsequent work shows that CRL exhibits strong exploration capabilities even when data collection is directed toward a single fixed goal~\citep{liu2025single,bastankhah2026demystifying}, and that increasing the depth of actor and critic networks further improves its goal-reaching performance~\citep{wang20261000}.
Despite these advances, it remains unclear whether exploration capability and network scaling alone are sufficient for CRL to solve substantially longer-horizon tasks.
Beyond policy learning, \citet{eysenbach2024inference} derive closed-form distributions over waypoint representations using contrastive learning.
However, their method learns representations from offline data without policy learning and requires a predefined number of waypoints, rather than adapting it to task difficulty.

Our method, \textbf{C-Slerp}, jointly learns contrastive representations and a goal-conditioned policy online.
Using the learned representations, it adapts the number of waypoints to each state-goal pair and infers their representations in closed form, without a separate high-level policy or graph search.
This allows us to examine how adapting the number of planning stages contributes to the benefits of hierarchy in long-horizon tasks.

\section{Preliminaries}
\label{sec:preliminaries}
We focus on fully online RL without offline data, where the replay buffer is continually updated during policy training and the critic is trained with a self-supervised objective without external rewards.
\paragraph{Goal-conditioned reinforcement learning (GCRL).}
We consider a goal-conditioned Markov decision process
$
\mathcal{M}_{g}
=
\bigl(
\mathcal{S},
\mathcal{A},
p,
p_0,
\gamma,
p_g,
r_g
\bigr),
$
where $\mathcal{S}$ and $\mathcal{A}$ denote the state and action spaces,
$p:\mathcal{S}\times\mathcal{A}\to\Delta(\mathcal{S})$ is the transition
kernel,
$p_0\!\in\!\Delta(\mathcal{S})$ is the initial-state distribution, and
$\gamma\in[0,1)$ is the discount factor.
A goal $g$ is sampled from the distribution
$p_g\!\in\!\Delta(\mathcal{G})$ over the goal space
$\mathcal{G}$, and a goal-conditioned policy
$\pi\!:\!\mathcal{S}\times\mathcal{G}\!\to\!\Delta(\mathcal{A})$
selects an action $a\sim\pi(\cdot\!\mid\! s,g)$ conditioned on the
current state $s$ and desired goal $g$.
For the analysis, we define the goal-reaching reward as $
r_g(s_t,a_t)
\triangleq
(1-\gamma)\,
p\!\left(
s_{t+1}=g
\mid
s_t,a_t
\right),
$
and the corresponding Q-function is
\begin{align}
Q_g^\pi(s,a)
\triangleq
\mathbb{E}_{\pi(\cdot\mid\cdot,g)}
\left[
\sum_{t'=t}^{\infty}
\gamma^{t'-t}
r_g(s_{t'},a_{t'})
\mid
s_t=s,\,
a_t=a
\right].
\label{eq:goal_conditioned_q}
\end{align}
Prior work \citep{eysenbach2022contrastive} has shown that this Q-function
can equivalently be expressed as the probability of reaching the goal under
the discounted state occupancy measure: $Q_g^\pi(s,a) = p_{\pi(\cdot\mid\cdot,g)}^{\gamma}(s_+=g\mid s,a).$
Here,
$p_{\pi(\cdot\mid\cdot,g)}^{\gamma}
(s_+=g\mid s,a)$
denotes the probability of reaching $g$ from $(s,a)$ at a temporal offset
drawn from $\operatorname{\textsc{Geom}}(1-\gamma)$ under the goal-conditioned policy 
\(\pi(\cdot\!\mid\!\cdot,g)\).\footnote{
More generally, for data collected under policies with \textit{different} commanded goals $g$,
we define the goal-averaged discounted occupancy as
$p_{\pi}^{\gamma}
(s_+=g'\mid s,a)
\triangleq
\int_{\mathcal{G}}
p_{\pi(\cdot\mid\cdot,g)}^{\gamma}
(s_+=g'\mid s,a)\,
p(g\mid s,a)\,\mathrm{d}g$.
}
\paragraph{Contrastive reinforcement learning (CRL).}
Instead of relying on Bellman backups, CRL~\citep{eysenbach2022contrastive} learns a critic for the
goal-conditioned Q-function $Q_g^\pi(s,a)$ in a self-supervised manner.
Given a batch $\mathcal{B}=\{(s_i,a_i,g_i)\}_{i=1}^{B}$,
where the goal $g_i$ corresponds to the state $\Delta_i$ steps
after $s_i$ and $\Delta_i\sim\operatorname{\textsc{Geom}}(1-\gamma)$,
the critic $f_{\phi,\psi}(s,a,g)$ is trained using the InfoNCE objective
\begin{equation}
\label{eq:crl-infonce-objective}
\mathcal{L}^{backward}(\phi,\psi;\mathcal{B})
=
-\frac{1}{B}
\sum_{i=1}^{B}
\log
\frac{
\exp\!\left(f_{\phi,\psi}(s_i,a_i,g_i)\right)
}{
\sum_{j=1}^{B}
\exp\!\left(f_{\phi,\psi}(s_j,a_j,g_i)\right)
}.
\end{equation}
Under the backward InfoNCE objective in \Cref{eq:crl-infonce-objective}, the Bayes-optimal critic can be expressed as
\begin{equation}
\label{eq:bayes-optimal-crl-critic}
f^*(s,a,g)
=
\log
\frac{
p_{\pi}^{\gamma}
(s_+=g\mid s,a)
}{
p_{\pi}^{\gamma}(g)\,C(g)
},
\end{equation}
where $p_{\pi}^{\gamma}(g)$
is the future-state marginal and $C(g)>0$ is a normalization term
that depends only on the future state $g$.
Since the denominator is independent of the action,
maximizing the critic $f^*(s,a,g)$ over actions is equivalent to
maximizing the discounted occupancy
$p_{\pi}^{\gamma}(s_+=g\mid s,a)$.
The policy $\pi_\theta(a\mid s,g)$ is therefore trained by maximizing
the critic:
\begin{equation}
\label{eq:crl-policy-objective}
\max_{\theta}\;
\mathbb{E}_{\substack{
s\sim p_{\mathcal D},\,g\sim p^{\gamma}_{\mathcal D}(\cdot\mid s)\\
a\sim\pi_\theta(\cdot\mid s,g)
}}
\left[
f_{\phi,\psi}(s,a,g)
\right],
\end{equation}
where $p_{\mathcal D}$ denotes the state distribution of the replay buffer and
$p^\gamma_{\mathcal D}(\cdot\mid s)$ the distribution of future goals from $s$
in the replay buffer, with temporal offsets drawn from
$\operatorname{\textsc{Geom}}(1-\gamma)$.
\section{Hyperspherical Contrastive Representations for Planning}
\label{sec:recursive_crl}
In this section, we look at planning a sequence of intermediate waypoints between a current state and a goal.
We approach this as an inference problem and show how the waypoints can be obtained in closed form by repeatedly composing contrastive representations.
We then extend this approach to planning over a \textit{variable} number of waypoints, deriving a principled criterion for selecting an appropriate number of waypoints for each state-goal pair using the same learned representations.

\subsection{Problem Formulation}
\label{sec:problem_formulation}

\paragraph{Setup.}
We collect data online under a hierarchical goal-conditioned policy $\hat{\pi}$ that combines a high-level planner with a policy $\pi_\theta$.
Every $h$ environment steps, the planner provides a waypoint
representation $z$ as a subgoal to reach the final goal,
and the policy $\pi_\theta(a\!\mid\! s,z)$ pursues it for the next $h$ steps.
We alternate between collecting data under $\hat{\pi}$ and updating the representations and policy.

We take a probabilistic approach to planning.
Given a state $s$ and goal $g$, we infer $K$ intermediate waypoints $w_{1:K}\!=\!(w_1,\ldots,w_K)$ along trajectories generated by $\hat{\pi}$ with goals drawn from $p_g$.
For notational consistency, we set $w_0=s$ and $w_{K+1}=g$.
We assume that consecutive waypoints $w_i$ and $w_{i+1}$ are
separated by an interval
$T_i\!\sim\!\operatorname{\textsc{Geom}}(1-\gamma)$,
drawn independently for $i=0,\ldots,K$.

\paragraph{Objective.}
Let $p_{\hat{\pi}}^{\gamma,K}(w_{1:K}\mid s,g)$ denote the distribution over $K$ intermediate waypoints given $s$ and $g$ under the data-generating process described above.
We approximate this posterior with a tractable distribution
$q(w_{1:K}\mid s,g)$ by minimizing the following KL divergence:
\begin{equation}
q^\star
=
\argmin_q
D_{\mathrm{KL}}
\left(
q(\cdot\mid s,g)
\;\middle\|\;
p_{\hat{\pi}}^{\gamma,K}(\cdot\mid s,g)
\right).
\label{eq:recursive_objective}
\end{equation}
For now, we fix $K$ and derive the corresponding optimal
waypoints. We then derive a criterion for adaptively selecting $K$
for each state-goal pair.

\subsection{Hyperspherical Contrastive Representations}
\label{sec:planning_representations}

To make the planning objective in \Cref{eq:recursive_objective} tractable, we parameterize temporal relationships using contrastive representations that lie on the unit hypersphere $\mathbb{S}^{d-1}$.
Our approach learns a state-only representation $\psi(s)\in\mathbb{S}^{d-1}$, a state-action representation $\phi(s,a)\in\mathbb{S}^{d-1}$, a rotation matrix $R$, and a policy $\pi_\theta(a\mid s,z)$, where $z\in\mathbb{S}^{d-1}$ is a representation of a goal.
We use these representations to infer intermediate waypoints for the policy $\pi_\theta$ to pursue.

\paragraph{Parameterization.}
We use $f_H(s,s^+)$ to score temporal relationships between states and
$f_L(s,a,s^+)$ to score future states given a current state-action pair.
To infer consecutive waypoints in a common latent space, $f_H$ uses
the \textit{same} encoder $\psi$ for both inputs.
Following the inner-product form used in CRL, we could score
state pairs using $f_H(u,v)=\psi(u)^\top\psi(v)$.
However, since this inner product is \textit{symmetric} in $u$ and $v$,
it cannot capture asymmetries in the dynamics.
We thus introduce a learnable rotation matrix $R$ to model these asymmetries while preserving the hyperspherical geometry.
Specifically, we parameterize $f_H$ and $f_L$ as
\begin{equation}
f_H(s,s^+) = \psi(s)^\top R\psi(s^+),
\qquad
f_L(s,a,s^+) = \phi(s,a)^\top\psi(s^+),
\end{equation}
where $R\in\mathbb{R}^{d\times d}$ satisfies $R^\top R=I$.\footnote{
Given a \textit{learnable} matrix $A\in\mathbb{R}^{d\times d}$, we define the
skew-symmetric matrix $\Omega=A-A^\top$ and apply the Cayley transform
$R=(I+\Omega)^{-1}(I-\Omega)$, which satisfies $R^\top R=I$ and thus
$\|Rx\|_2=\|x\|_2$ for all $x\in\mathbb{R}^d$.
}
This allows us to score an entire waypoint sequence by repeatedly applying a \textit{single} $f_H$ to consecutive waypoint pairs.
For $f_L$, we follow the standard CRL parameterization and train
the policy to maximize $f_L(s,a,s^+) =\phi(s,a)^\top\psi(s^+)$
as in \cref{eq:crl-policy-objective}.
The corresponding learning objectives are described in
\Cref{sec:representation_policy_learning}.

\subsection{Closed-Form Waypoint Inference}
\label{sec:closed_form_waypoint_inference}

\paragraph{Key assumptions.}
Our derivation relies on the following assumptions, with theoretical justification based on prior work~\citep{wang2020understanding,eysenbach2024inference} provided in Appendix~\ref{app:assumption}:
\begin{assumption}[Bayes optimality]
\label{as:critic}
The learned $f_H$ under the symmetric InfoNCE objective has the following Bayes-optimal form\textup{:}
\begin{equation}
f_H(u,v)
=
\log
\frac{
p_{\pi}^{\gamma}(u^+ = v\mid u)
}{
p_{\pi}^{\gamma}(v)\, C
},
\end{equation}
where $\pi$ is the data-collection policy and $C>0$ is a constant independent of $u$ and $v$.
\end{assumption}

\begin{assumption}[Uniformity]
\label{as:uniform}
The marginal distribution over the representations learned for $f_H$ under the symmetric InfoNCE objective is uniform on the unit hypersphere $\mathbb{S}^{d-1}$\textup{:}
\begin{equation}
p_\pi^\gamma(\psi)
\triangleq
\int
p_\pi^\gamma(v)\,
\delta\!\left(\psi-\psi(v)\right)
\,\mathrm{d}v
=
\operatorname{Unif}\!\left(\mathbb{S}^{d-1}\right).
\end{equation}
\end{assumption}

\paragraph{Tractable reformulation.}
To obtain a tractable objective, we reformulate \Cref{eq:recursive_objective} in terms of the learned representations.
We first factorize $q$ as $q(w_{1:K}\!\mid\! s,g)=\prod_{i=1}^{K}q_i(w_i\!\mid\! s,g)$, with one distribution $q_i$ per waypoint $w_i$.
This allows us to optimize $q$ for different $K$ within the same formulation below.
Under \Cref{as:critic} and the Markov assumption, the objective can be equivalently written as follows, with the derivation provided in Appendix~\ref{app:reformulation}:
\begin{align}
q^\star
&=
\argmin_q
D_{\mathrm{KL}}
\left(
q(\cdot\mid s,g)
\;\middle\|\;
p_{\hat{\pi}}^{\gamma,K}(\cdot\mid s,g)
\right)
\notag\\
&=
\argmax_q
\mathbb{E}_{w_{1:K}\sim q(\cdot\mid s,g)}
\left[
\sum_{i=0}^{K} f_H(w_i,w_{i+1})
\right]
-
\sum_{i=1}^{K}
D_{\mathrm{KL}}
\left(
q_i(\cdot\mid s,g)
\;\middle\|\;
p_{\hat{\pi}}^\gamma(\cdot)
\right)
\label{eq:waypoint-critic-objective}
\end{align}
This reformulation provides a natural interpretation of our objective in \Cref{eq:recursive_objective}: it maximizes the expected sum of $f_H$, which scores temporal relationships between waypoints, while encouraging each waypoint distribution $q_i$ to remain close to the marginal distribution of states visited under $\hat{\pi}$.

However, inferring waypoints directly in the high-dimensional state space is computationally challenging.
Since the policy $\pi_\theta(a\!\mid\! s,z)$ depends on the waypoint $w$ only through its representation $z=\psi(w)$, inferring its \textit{representation} is sufficient to provide a subgoal to the policy.
We therefore infer distributions $q_i(z_i\!\mid\! s,g)$ over representations $z_i\!=\!\psi(w_i)$ rather than the raw states $w_i$.
\Cref{as:uniform} and the parameterization of $f_H$ then yield the tractable objective $\mathcal{J}(q)$ to maximize:
\begin{equation}
\mathcal{J}(q)
=
\mathbb{E}_{z_{1:K}\sim q(\cdot\mid s,g)}
\Bigg[
\psi(s)^\top R z_1
+
\sum_{i=2}^{K}
z_{i-1}^\top R z_i
+
z_K^\top R\,\psi(g)
\Bigg]
+
\sum_{i=1}^{K}
\mathcal{H}(q_i(\cdot \mid s,g)),
\label{eq:latent_objective}
\end{equation}
where $\mathcal{H}$ denotes entropy.
This objective links the current state $s$ to the goal $g$ by repeatedly aligning consecutive representations $z_i$ through the rotation matrix $R$, while the entropy term encourages stochasticity in the sequence of representations.
The full derivation is provided in Appendix~\ref{app:reformulation}.

\paragraph{Closed-form solution.}
To model the stochasticity encouraged by $\mathcal{J}(q)$, we parameterize each $q_i(\cdot\mid s,g)$ as the distribution of a Gaussian sample normalized onto the unit hypersphere:
\[
\tilde z_i \sim \mathcal{N}\!\left(\mu_i(s,g),\,\sigma_i^2(s,g)I\right),
\qquad
z_i = \frac{\tilde z_i}{\|\tilde z_i\|_2}
\sim q_i(\cdot\mid s,g),
\]
with $\|\mu_i(s,g)\|_2=1$.
As in SAC~\citep{haarnoja2018soft}, we weight the entropy term in \Cref{eq:latent_objective} by $\alpha>0$ to control stochasticity, with $\alpha=1$ recovering the original objective.
Under a small-variance approximation, we obtain the following optimal solution for maximizing $\mathcal{J}(q)$:
\begin{equation}
\left\{
\begin{aligned}
\mu_{i}^{\,*}(s,g) \:&= (R^T)^{i} \; \mathrm{slerp}\!\left(\psi(s),\, R^{K+1}\psi(g),\, \tfrac{i}{K+1}\right),\\[6pt]
\:\:\sigma_i^{*2}(s,g) &=
\frac{\alpha}{
2 \cos(\frac{\theta}{K+1})
}
\end{aligned}
\right.
\quad
\text{for } i = 1,\dots,K,
\label{eq:closed_form_high_policy}
\end{equation}
where $\theta=\arccos(\psi(s)^\top R^{K+1}\psi(g))$
and $\operatorname{slerp}$ denotes spherical linear interpolation. The derivation is provided in Appendix~\ref{app:slerp}.

When $R=I$, the solution places all means $\mu_{i}$ evenly along the \textit{geodesic} between $\psi(s)$ and $\psi(g)$ on the unit hypersphere.
In contrast, our method accounts for temporal transitions encoded in $R$ by interpolating from $\psi(s)$ toward the \textit{rotated} goal $R^{K+1}\psi(g)$ and mapping each interpolated point back using $(R^\top)^i$ (see \Cref{fig:fig2} for a geometric interpretation).
Importantly, these closed-form expressions allow us to infer distributions over waypoint representations \textit{directly} from the learned contrastive representations without training an additional neural network.

\begin{figure}[t]
    \centering
    \vspace{-20pt}
    \includegraphics[width=0.98\linewidth]{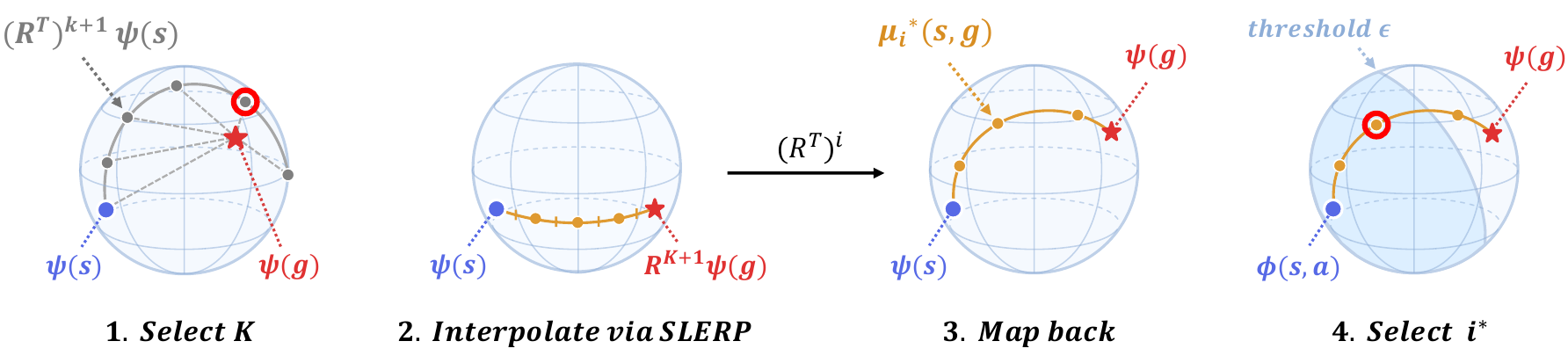}
    \vspace{-6pt}
    \caption{\textbf{Geometric interpretation of C-Slerp.}
    The number of waypoints $K$ is selected using the alignment between $(R^\top)^{k+1}\psi(s)$ and $\psi(g)$.
    Given the selected $K$, waypoint mean directions are obtained by spherical interpolation from $\psi(s)$ to $R^{K+1}\psi(g)$, followed by mapping back via $(R^\top)^i$.
    The policy then pursues the furthest waypoint whose alignment with $\phi(s,a)$ meets the threshold $\epsilon$.}
    \vspace{-0pt}
    \label{fig:fig2}
\end{figure}

\subsection{How Many Waypoints Are Needed for Each State-Goal Pair?}
\label{k-selection}
Since the closed-form solution in \Cref{eq:closed_form_high_policy} applies to any number of waypoints $K$, we now consider how to select $K$ for each state-goal pair.
Recall that $f_H(s,g)$ measures alignment between $R^\top\psi(s)$
and $\psi(g)$ to score a single geometric transition from $s$ to $g$.
Extending this to $K+1$ geometric transitions, we define the following
alignment score between $(R^\top)^{K+1}\psi(s)$ and $\psi(g)$:
\begin{equation}
\ell_K(s,g)
\triangleq
\left((R^\top)^{K+1}\psi(s)\right)^\top\psi(g)
=
\psi(s)^\top R^{K+1}\psi(g).
\label{eq:k_selection_score}
\end{equation}
This score measures how well the goal representation aligns with the direction predicted from the current state after $K+1$ geometric transitions under $\hat{\pi}$.
A larger score indicates that the corresponding number of waypoints is more compatible with the state-goal pair.
To infer $K$ probabilistically, we model $\psi(g)$ with a
von Mises--Fisher likelihood centered on the predicted direction $(R^\top)^{K+1}\psi(s)$:
\begin{equation}
p(\psi(g)\mid\psi(s),K)
\triangleq
C_d(\kappa)\exp\!\left(\kappa\ell_K(s,g)\right),
\label{eq:k_selection_likelihood}
\end{equation}
where $\kappa>0$ is the concentration parameter and $C_d(\kappa)$ is the normalizing constant.
With a uniform prior over $K\in\{0,\ldots,K_{\max}\}$, Bayes' rule yields
\begin{equation}
p(K\mid\psi(s),\psi(g))
=
\operatorname{Softmax}_{K\in\{0,\ldots,K_{\max}\}}
\left(
\frac{\ell_K(s,g)}{\eta}
\right),
\label{eq:k_selection_distribution}
\end{equation}
where $\eta=1/\kappa$ controls the sharpness of the distribution over $K$.
This allows us to infer the number of waypoints for each state-goal pair using only the learned contrastive representations.
The theoretical justification for the alignment score $\ell_K$ and the full derivation are provided in Appendix~\ref{app:k-selection}.

\paragraph{Which waypoint should the policy pursue?}
After selecting $K$, we can sample waypoint representations $z_{1:K}$ from the closed-form distributions in \Cref{eq:closed_form_high_policy}.
The remaining question is \textit{which} waypoint representation $z_i$ to provide to the policy, given its current goal-reaching capabilities.

Our learned $f_L(s,a,g)=\phi(s,a)^\top\psi(g)$ provides a natural proxy for the goal-reaching capability of the current policy, based on its probabilistic interpretation in \Cref{eq:bayes-optimal-crl-critic}.
Since $z_i$ lies in the same space as $\psi(g)$, we use its alignment with $\phi(s,a)$ to assess how likely the policy is to reach it.
For each candidate $z_i$, including $z_{K+1}=\psi(g)$, we sample $a_i\sim\pi_\theta(\cdot\mid s,z_i)$ and compute $\rho_i=\phi(s,a_i)^\top z_i$.
We select the waypoint with the largest index whose score meets a threshold $\epsilon\in[-1,1]$:
\begin{equation}
i^*
=
\max\left\{
i\in\{1,\ldots,K+1\}:\rho_i\ge\epsilon
\right\}.
\label{eq:waypoint_index_selection}
\end{equation}
If no candidate meets the threshold, we set $i^*=1$.
This allows the policy $\pi_\theta(a\mid s,z_{i^*})$ to make progress toward the goal by pursuing the furthest waypoint estimated to be reachable, avoiding overly conservative subgoals.
We provide a geometric interpretation of selecting $K$ and $i^*$ in \Cref{fig:fig2} and visualizations of the $\rho_i$ values in \Cref{fig:waypoint_energy}.

\subsection{Learning Objectives and Algorithm}
\label{sec:representation_policy_learning}

\paragraph{Representation learning.}
To learn the representations used for planning, we sample a batch $\mathcal{B}=\{(s_i,a_i,s_i^+)\}_{i=1}^{B}$ from the replay buffer, with the temporal interval between $s_i$ and $s_i^+$ drawn from $\operatorname{\textsc{Geom}}(1-\gamma)$.
Using this batch, we train $f_L$ with the backward InfoNCE objective in~\cref{eq:crl-infonce-objective} and $f_H$ with the symmetric InfoNCE objective defined in~\Cref{eq:symmetric-infonce-objective}.
The total representation loss is
\begin{equation}
\label{eq:repr_loss}
\mathcal{L}_{\mathrm{repr}}
=
\mathbb{E}_{\mathcal{B}}
\left[
\mathcal{L}^{\mathrm{backward}}(\phi,\psi;\mathcal{B})
+
\lambda_{\mathrm{repr}}
\mathcal{L}^{\mathrm{symmetric}}(R,\psi;\mathcal{B})
\right],
\end{equation}
where the expectation is over batches $\mathcal{B}$ sampled as described above.

\paragraph{Policy learning.}
Following CRL, we train the policy $\pi_\theta$ to maximize
$f_L(s,a,g)=\phi(s,a)^\top\psi(g)$:
\begin{equation}
\max_\theta\;
\mathbb{E}_{\substack{
s\sim p_{\mathcal D},\,g\sim p^{\gamma}_{\mathcal D}(\cdot\mid s)\\
a\sim\pi_\theta(\cdot\mid s,\psi(g))
}}
\left[\phi(s,a)^\top\psi(g)\right].
\label{eq:policy_objective}
\end{equation}
Specifically, the policy $\pi_\theta(a\mid s,\psi(g))$ is conditioned on the goal representation $\psi(g)$, enabling it to pursue waypoints $z$ inferred in the representation space induced by $\psi$.

\paragraph{Algorithm.}
Every $h$ environment steps, the high-level planner determines the number of waypoints for the current state-goal pair and selects a waypoint as a subgoal for the policy.
Algorithm~\ref{alg:aplanning} provides the full pseudocode, including data collection and policy and representation learning.

\section{Experiments} \label{sec:experiments}
\begin{figure}[t]
    \centering
    \vspace{-5pt}
    \includegraphics[width=0.98\linewidth]{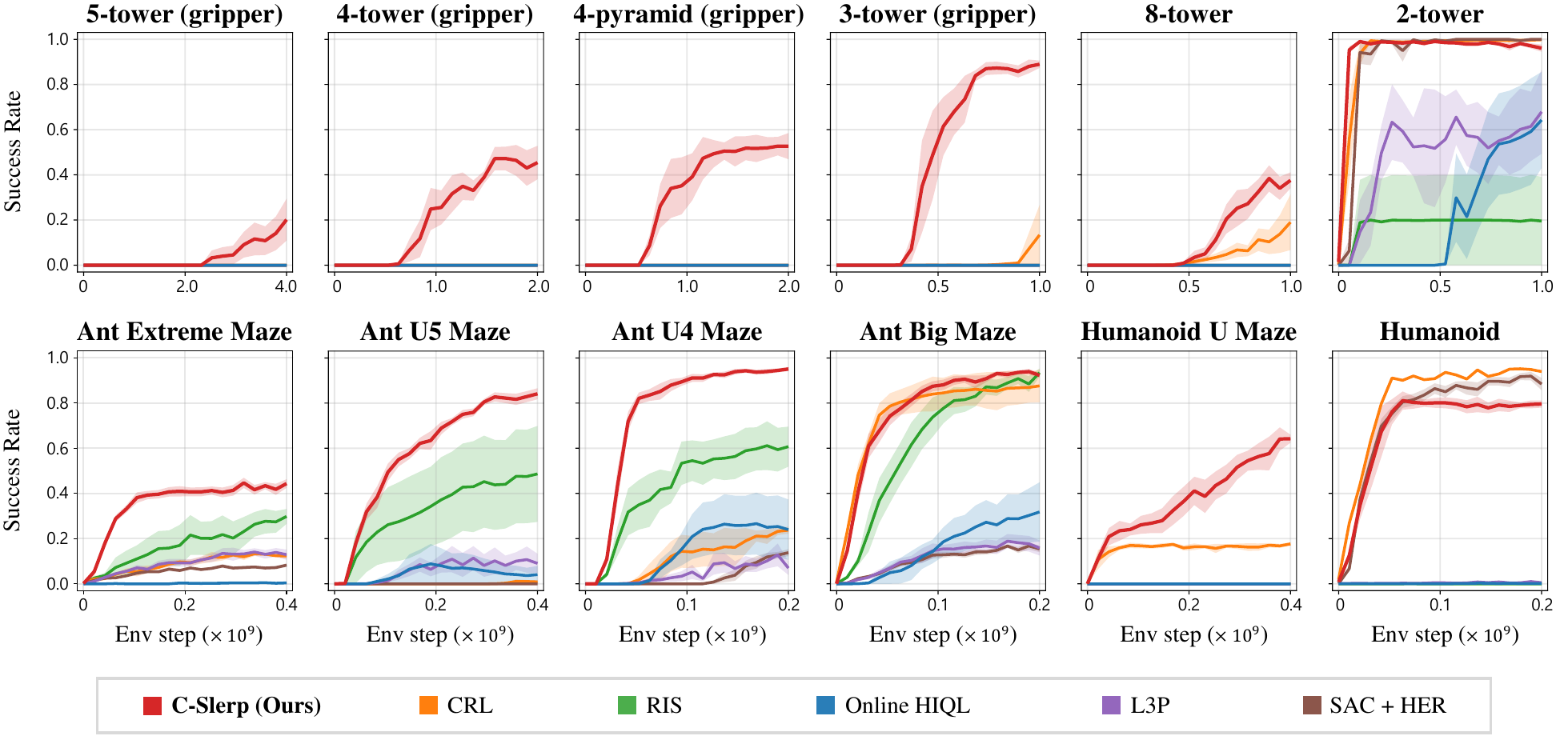}
    \vspace{-2.5pt}
    \caption{\textbf{C-Slerp solves long-horizon tasks where baselines never achieve a single success.}
    Learning curves show success rates on BuilderBench manipulation tasks
    (\textit{top}) and JaxGCRL locomotion tasks (\textit{bottom}).
    C-Slerp achieves the highest success rates on most tasks across both benchmarks and is the \textit{only} method to succeed on all cube-stacking tasks with a gripper.}
    \vspace{-10pt}
    \label{fig:fig3}
\end{figure}

\paragraph{Benchmarks.}
We evaluate C-Slerp on six locomotion tasks in JaxGCRL~\citep{bortkiewicz2025accelerating} and six manipulation tasks in BuilderBench~\citep{ghugare2025builderbench}.
The JaxGCRL tasks cover Ant and Humanoid locomotion, including navigation through mazes of varying sizes to test goal-reaching capability over different horizons.
In BuilderBench, the agent arranges cubes into target configurations using a gripper or applying forces \textit{directly} to the cubes~\citep{yanscaling}.
The number of cubes determines the difficulty of these tasks, with more cubes requiring longer horizons.
Using a \textit{gripper} extends these horizons further, as each placement requires additional interactions to grasp and release the cube.
Further details on both benchmarks are provided in Appendix~\ref{sec:environments}, with environment visualizations in Figure~\ref{fig:env_explanation}.

\paragraph{Baselines.}
We compare \textbf{C-Slerp} with five GCRL methods:
As a baseline using contrastive representations, we include \textbf{CRL}~\citep{eysenbach2022contrastive}, which learns representations of state-action pairs and future states and trains a policy to maximize their alignment.
To compare different uses of hierarchy, we include \textbf{RIS}~\citep{chane-sane2021ris} and \textbf{Online HIQL}~\citep{park2023hiql}.
RIS uses imagined subgoals to regularize a flat policy during training, whereas Online HIQL uses a high-level policy to produce latent subgoals that a low-level policy pursues.
For this comparison, we adapt HIQL, originally designed for offline reinforcement learning, to the online setting.
For graph-based planning, we include \textbf{L3P}~\citep{zhang2021l3p}, which constructs a graph of latent landmarks connected by reachability estimates and searches it to plan waypoint sequences.
Finally, \textbf{SAC + HER}~\citep{haarnoja2018sac,andrychowicz2017hindsight} is a standard GCRL baseline that relabels goals with achieved states to learn from unsuccessful trajectories. Implementation details and hyperparameter settings are provided in Appendix~\ref{sec:implementation_details}.

\subsection{Experimental Results}
\paragraph{Benchmark Performance.}
Figure~\ref{fig:fig3} compares C-Slerp with five baselines across twelve benchmark tasks.
C-Slerp achieves the highest success rates on most tasks across both benchmarks.
In BuilderBench, the agent must learn to construct target structures without a goal curriculum or intermediate rewards.
Building taller towers requires more placements, thus leading to longer horizons and making exploration more challenging.
The results show that C-Slerp and CRL perform comparably on shorter-horizon tasks such as the 2-tower task, but the performance gap between C-Slerp and the baselines widens as the number of cubes increases or when a gripper is used.
Notably, C-Slerp is the \textit{only} method to succeed on all cube-stacking tasks with a gripper, while most baselines fail to achieve any success.

In navigation tasks, several baselines also perform comparably on shorter-horizon tasks such as Ant Big Maze and Humanoid.
As the required horizon increases, however, C-Slerp achieves consistently higher success rates.
These results suggest that C-Slerp not only enables long-horizon planning but also facilitates \textit{exploration}, allowing the agent to discover successful behaviors more efficiently by decomposing challenging tasks into easier subproblems.

\begin{figure*}[h]
    \centering
    \begin{minipage}[t]{0.5\textwidth}
        \centering
        \vspace{-2.5pt}
        \includegraphics[width=0.99\linewidth]{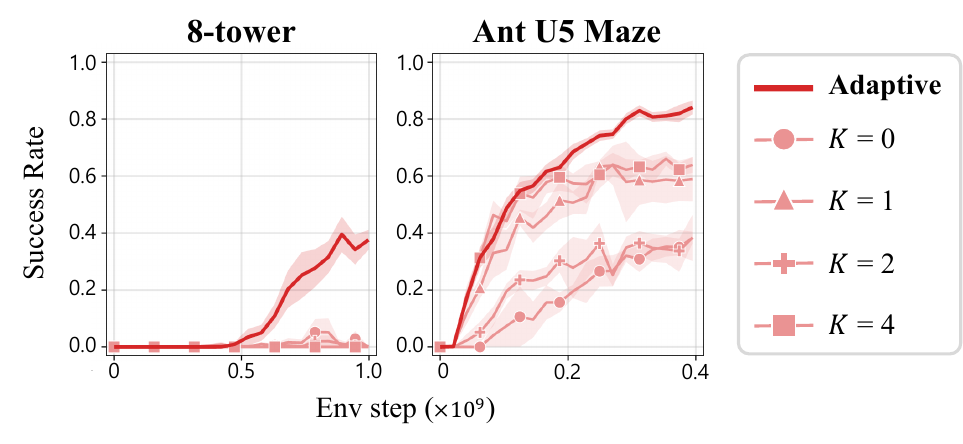}
        \vspace{-15pt}
        \caption{Ablation on $K$-selection}
        \vspace{-17.5pt}
        \label{fig:fig4}
    \end{minipage}%
    \begin{minipage}[t]{0.5\textwidth}
        \centering
        \vspace{-2.5pt}
        \includegraphics[width=0.99\linewidth]{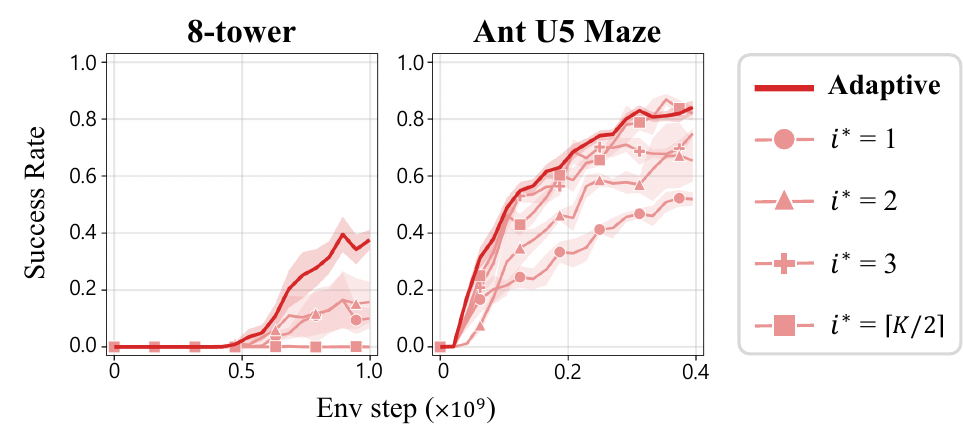}
        \vspace{-15pt}
        \caption{Ablation on $i$-selection}
        \vspace{-17.5pt}
        \label{fig:fig5}
    \end{minipage}
\end{figure*}

\paragraph{Is \textit{adaptive} waypoint planning necessary for long-horizon tasks?}
We first examine whether adapting the number of waypoints $K$ (as described in \Cref{k-selection}) improves performance over using a fixed number.
\Cref{fig:fig4} compares adaptive and fixed choices of $K$ on the 8-tower task and Ant U5 Maze.
For each fixed-$K$ variant, the number of waypoints remains constant throughout the episode.
For $K=0$, no intermediate waypoints are used, and thus the policy is always conditioned directly on the goal representation $\psi(g)$.
The results show that adaptive selection achieves the highest success rates in both environments.
The fixed-$K$ variants make little progress on the 8-tower task and reach lower final performance on Ant U5 Maze.
Furthermore, increasing the fixed number of waypoints does not consistently improve performance, suggesting that simply introducing \textit{more} intermediate waypoints is insufficient.

These results suggest that the number of waypoints should \textit{match} the difficulty of the goal-reaching problem.
Too few waypoints can make each subproblem too difficult for the policy to solve, whereas too many may slow progress toward the goal by introducing unnecessary planning stages.
Moreover, the appropriate number of waypoints can also change \textit{within} an episode as the agent approaches the goal (see~\Cref{fig:KI_selection,fig:KI_selection_2}).
Consequently, a fixed number of waypoints that works well early in an episode may become unnecessarily large as the remaining task becomes easier.
Adaptive selection accounts for this change by adjusting $K$ to the current state-goal pair, highlighting that the benefit of C-Slerp comes not only from \textit{introducing} intermediate waypoints but also from \textit{adapting} their number to the problem.

We next examine how to select the waypoint for the policy to pursue.
Figure~\ref{fig:fig5} compares adaptive waypoint selection with fixed choices of $i$, while keeping the selection of $K$ adaptive.
Among the fixed selection rules, $i=1$ and $i=2$ perform better on the 8-tower task, whereas $i = \lceil K/2 \rceil$ and $i=3$ perform better on Ant U5 Maze.
Thus, a choice that works well in one environment may not be the best choice in another.
Earlier waypoints (e.g., $i=1$) provide more conservative subgoals but may slow progress, whereas later waypoints can enable greater progress but may be too difficult for the current policy to reach.
Adaptive selection uses the learned $f_L$ to make this choice and achieves the best performance in both environments (see~\Cref{fig:waypoint_energy} for a visualization of the selected waypoints).
Together, these ablations show that adaptive selection of both $K$ and $i$ is needed to achieve the full performance of C-Slerp on long-horizon tasks.

\begin{figure*}[ht]
    \centering
    \begin{subfigure}{0.365\textwidth}
        \centering
        \vspace{-15pt}
        \includegraphics[width=\linewidth]{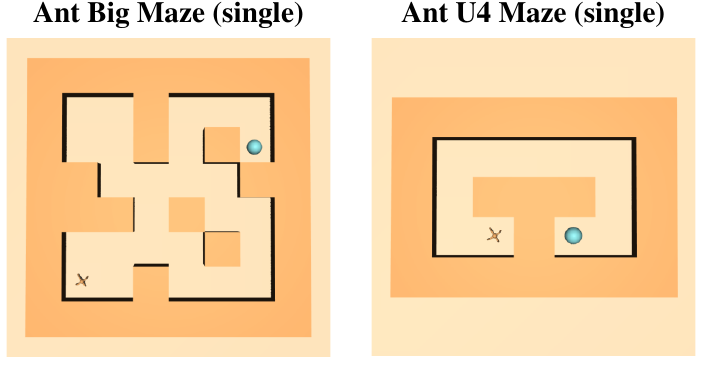}
        \caption{Single-goal environments}
        \label{fig:structure_targets}
    \end{subfigure}
    \hspace{0.05\textwidth}
    \begin{subfigure}{0.512\textwidth}
        \centering
        \vspace{-15pt}
        \includegraphics[width=\linewidth]{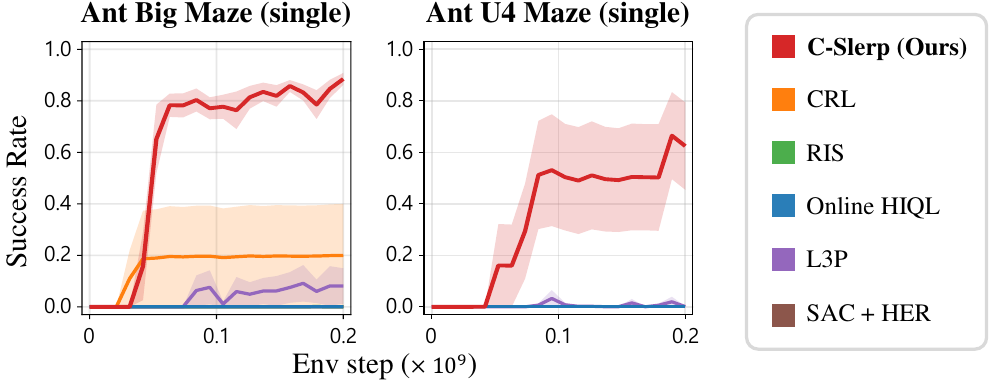}
        \caption{Learning curves}
        \label{fig:structure_results}
    \end{subfigure}
    \caption{\textbf{Evaluation of C-Slerp in single-goal environments.} To assess exploration capability, we evaluate C-Slerp with a single fixed goal throughout training. The environments and their fixed goals are shown in (a), and success rates for C-Slerp and the baselines are presented in (b).}
    \vspace{-0pt}
    \label{fig:singlegoal}
\end{figure*}

\paragraph{Exploration in single-goal settings.}
We further evaluate C-Slerp in a single-goal setting, where a \textit{single} fixed goal is provided throughout training and evaluation, following the setting considered in prior work~\citep{liu2025single,bastankhah2026demystifying}.
As shown in Figure~\ref{fig:singlegoal}, we modify the goal distributions of Ant Big Maze and Ant U4 Maze by fixing the goal to the most \textit{distant} location in each maze environment.
C-Slerp achieves higher success rates than all baselines, with all baselines except CRL and L3P achieving zero success.
These results suggest that C-Slerp can facilitate \textit{directed} exploration by producing intermediate waypoints that guide the agent toward distant goals, even when only a single fixed goal is provided.

\begin{figure}[h]
    \centering
    \vspace{7.5pt}
    \includegraphics[width=1.00\linewidth]{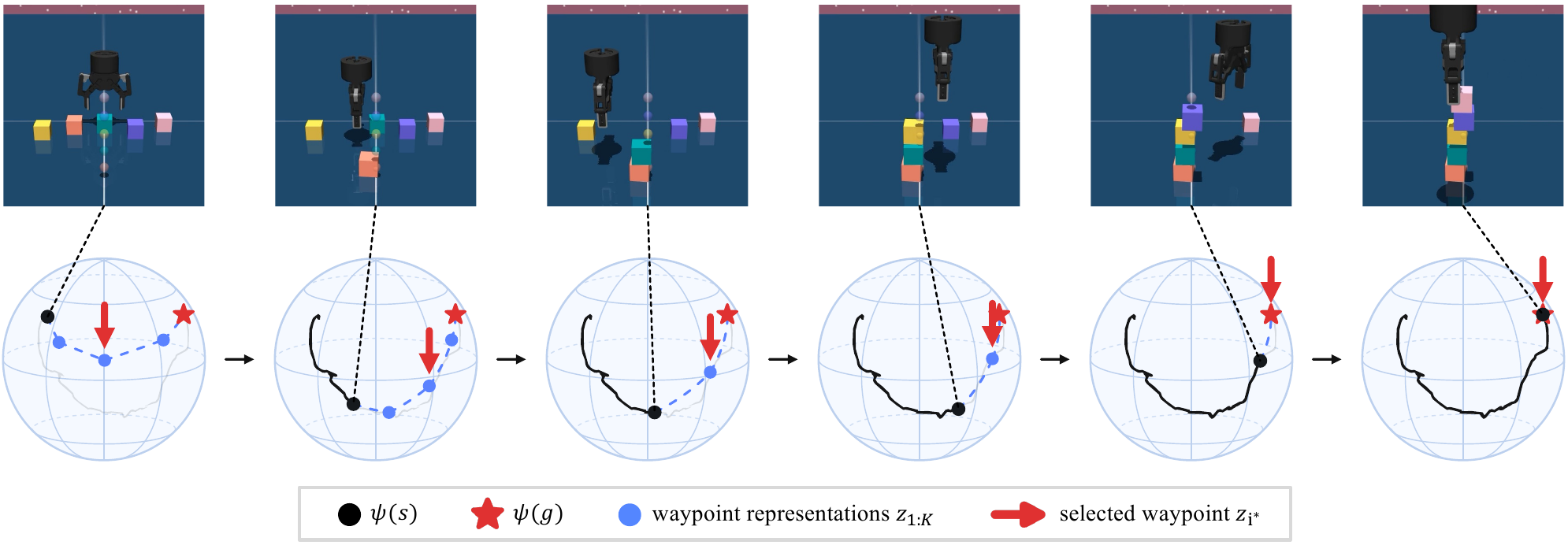}
    \vspace{-12.5pt}
    \caption{\textbf{Adaptive planning in C-Slerp.}
    The top row shows successive states in the 5-tower task with a gripper.
    The bottom row shows spherical PCA~\citep{liu2019spherical} projections of the latent trajectory, inferred waypoints, and selected subgoal.
    As the agent approaches the goal, it \textit{replans} by adjusting the number of waypoints and selecting a waypoint to pursue. We provide complete rollout visualizations in \Cref{fig:rollout_1,fig:rollout_2} and videos on our \href{https://jonghaepark.github.io/c-slerp}{\textcolor{magenta}{project page}}.}
    \vspace{-0.0pt}
    \label{fig:fig7}
\end{figure}

\paragraph{How do the inferred waypoints relate to task progress?}
\Cref{fig:fig7} shows successive states in the 5-tower task with a gripper, alongside projections of their representations and inferred waypoints.
As the agent builds the tower, its state representation $\psi(s)$ progresses along the latent trajectory toward the goal representation $\psi(g)$.
At each replanning step, C-Slerp updates the number of intermediate waypoints $K$, infers their representations, and selects one for the policy to pursue.
Notably, the selected waypoint is not always the \textit{first} in the inferred plan, which is consistent with our adaptive waypoint selection strategy described in \Cref{k-selection}.
We can also observe that fewer waypoints are used as the tower nears completion, suggesting that the selected number of waypoints $K$ reflects the remaining temporal distance to the goal.
Together, these visualizations show how C-Slerp guides the policy toward the goal by adapting the number of waypoints and selecting a subgoal in the learned representation space.
Additional rollout visualizations are provided in Appendix~\ref{app:add_exp}.

\section{Conclusion}

We introduce C-Slerp, an online goal-reaching method that uses hyperspherical contrastive representations for adaptive planning.
These representations enable closed-form waypoint inference through spherical linear interpolation, without a separate high-level network or graph search.
Experiments show that C-Slerp solves long-horizon navigation and manipulation tasks, including settings where none of the baselines achieve any success.
Ablations further show that adaptively selecting both the number of waypoints and the waypoint for the policy to pursue improves performance over fixed choices. These results suggest that adapting task decomposition to the goal-reaching problem contributes to the benefits of hierarchy in goal-conditioned reinforcement learning.

Our approach has limitations.
Like other self-supervised RL methods, C-Slerp requires extensive environment interaction for training, which can be costly to collect on real robots.
This motivates extending C-Slerp to an offline RL setting, where it could learn from offline data without additional environment interaction.
Also, our experiments use state observations in simulation, leaving learning from visual observations and deployment on real robots untested.
Improving data efficiency and extending the method to pixel-based observations are important directions for future work.

\paragraph{Final remarks.}
In this paper, we show how choosing an appropriate amount of planning can unlock new capabilities for reinforcement learning.
C-Slerp learns to stack five cubes with a gripper without a goal curriculum or dense rewards, while all evaluated baselines fail to achieve a single success.
These findings suggest that adapting how tasks are decomposed is a promising direction for extending online goal-conditioned reinforcement learning to more challenging problems.
Additionally, although C-Slerp can be viewed as a hierarchical goal-conditioned RL method, it achieves these benefits \textit{without} a separate high-level policy network or a planning graph.
Instead, planning follows directly from spherical linear interpolation in the learned representation space, showing that adaptive planning need not require a complex architecture.

\subsubsection*{Acknowledgments}
We thank Mahsa Bastankhah and Royina Jayanth for their helpful feedback on the manuscript. We also thank members of the Lab for Autonomous Robotics Research (LARR) and the Princeton RL Lab for insightful discussions.

\newpage
\bibliography{iclr2027_conference}
\bibliographystyle{iclr2027_conference}

\appendix

\clearpage
\appendix
\startcontents[appendix]

\begingroup
\color{black}
\hypersetup{hidelinks}

\section*{\Large Appendix}
\printcontents[appendix]{}{1}[2]{%
    \linespread{1.2}\selectfont
}

\endgroup

\clearpage 
\section{Theoretical Results}

\subsection{Justification of the key assumptions}
\label{app:assumption}

\newtheorem*{bayesassumption}{Assumption~\ref{as:critic}}
\begin{bayesassumption}[Bayes optimality]
The learned $f_H$ under the symmetric InfoNCE objective has the following Bayes-optimal form\textup{:}
\begin{equation}
f_H(u,v)
=
\log
\frac{
p_{\pi}^{\gamma}(u^+ = v\mid u)
}{
p_{\pi}^{\gamma}(v)\, C
},
\label{eq:bayesform}
\end{equation}
where $\pi$ is the data-collection policy and $C>0$ is a constant independent of $u$ and $v$.
\end{bayesassumption}

The symmetric InfoNCE objective for a batch
$\mathcal{B}=\{(s_i,a_i,s_i^+)\}_{i=1}^{B}$ is defined as
\begin{equation}
\mathcal{L}^{\mathrm{symmetric}}(R,\psi;\mathcal{B})
\!\triangleq
-\frac{1}{B}\!\sum_{i=1}^{B}\!
\left[
\log
\frac{\exp(f_H(s_i,s_i^+))}
{\sum_{j=1}^{B}\exp(f_H(s_i,s_j^+))}
\!+\!
\log
\frac{\exp(f_H(s_i,s_i^+))}
{\sum_{j=1}^{B}\exp(f_H(s_j,s_i^+))}
\right]\!\!.
\label{eq:symmetric-infonce-objective}
\end{equation}

The Bayes-optimal form in \Cref{eq:bayesform} follows directly from Lemma~5 of \citet{eysenbach2024inference}.
Their formulation uses denominators without resubstitution~\citep{poole2019variational,yeh2022decoupled}, excluding the positive pair from each denominator.
Since $|f_H|\le 1$, including the positive pair changes the loss by at most $2\log(1+e^2/(B-1))$, which vanishes as $B\to\infty$.
The same limiting result therefore applies to \cref{eq:symmetric-infonce-objective}.
We assume that $f_H$ attains this optimum in our analysis.
\\
\\

\newtheorem*{uniformassumption}{Assumption~\ref{as:uniform}}
\begin{uniformassumption}[Uniformity]
The marginal distribution over the representations learned for $f_H$ under the symmetric InfoNCE objective is uniform on the unit hypersphere $\mathbb{S}^{d-1}$\textup{:}
\begin{equation}
p_\pi^\gamma(\psi)
\triangleq
\int
p_\pi^\gamma(v)\,
\delta\!\left(\psi-\psi(v)\right)
\,\mathrm{d}v
=
\operatorname{Unif}\!\left(\mathbb{S}^{d-1}\right).
\end{equation}
\end{uniformassumption}
Following \citet{wang2020understanding}, we decompose the objective in \cref{eq:symmetric-infonce-objective} into alignment and uniformity terms corresponding to the numerators and denominators, respectively. We justify Assumption~\ref{as:uniform} using this decomposition and its entropy interpretation.
Since $R$ is an orthogonal matrix and the representations $\psi$ have unit norm, we have
\[
f_H(s_i,s_i^+)
=
\psi(s_i)^\top R\psi(s_i^+)
=
1-\frac12\|R^\top\psi(s_i)-\psi(s_i^+)\|_2^2.
\]
The alignment term therefore pulls $R^\top\psi(s_i)$ toward $\psi(s_i^+)$ for each positive pair.
For the uniformity term, we consider the log-denominator of the first term
in \cref{eq:symmetric-infonce-objective}.
If positive pairs are sufficiently aligned, replacing $R^\top\psi(s_i)$ with $\psi(s_i^+)$ in the von Mises--Fisher (vMF) log-density gives the approximation below.
After subtracting the constant $\log B$, this log-denominator can be written as
\begin{align}
\frac{1}{B}\sum_{i=1}^{B}
\log\!\left(
\frac{1}{B}\sum_{j=1}^{B}
\exp\!\left(f_H(s_i,s_j^+)\right)
\right)\!
&=
\frac{1}{B}\sum_{i=1}^{B}
\log\!\left(
\frac{1}{B}\sum_{j=1}^{B}
\frac{
\exp\!\left(\psi(s_i)^\top R\psi(s_j^+)\right)
}{Z_{\mathrm{vMF}}}
\right)
\!\!+\log Z_{\mathrm{vMF}}
\notag\\
&=
\frac{1}{B}\sum_{i=1}^{B}
\log\widehat p_{\mathrm{vMF}}(R^\top\psi(s_i))
+\log Z_{\mathrm{vMF}}
\notag\\
&\approx
\frac{1}{B}\sum_{i=1}^{B}
\log\widehat p_{\mathrm{vMF}}(\psi(s_i^+))
+\log Z_{\mathrm{vMF}}
\notag\\
&=
-\widehat{\mathcal H}(\psi)
+\log Z_{\mathrm{vMF}},
\label{eq:uniformity-entropy-interpretation}
\end{align}
where $\widehat p_{\mathrm{vMF}}(u)=\bigl[\sum_j\exp\!\left(u^\top\psi(s_j^+)\right)\bigr]/(BZ_{\mathrm{vMF}})$ is a vMF kernel density estimate, $Z_{\mathrm{vMF}}$ is its normalizing constant, and $\widehat{\mathcal H}(\psi)=-\bigl[\sum_i\log\widehat p_{\mathrm{vMF}}(\psi(s_i^+))\bigr]/B$ is the corresponding resubstitution entropy estimate.
The entropy interpretation for the denominator of the second term in \cref{eq:symmetric-infonce-objective} can be derived in the same way.

Minimizing the term $-\widehat{\mathcal H}(\psi)$ in \cref{eq:uniformity-entropy-interpretation} maximizes the estimated marginal entropy of $\psi$, encouraging representations to spread across the sphere.
Since maximizing entropy on the unit sphere yields the uniform distribution on $\mathbb{S}^{d-1}$, we assume a uniform marginal distribution for the learned representations $\psi$ in our analysis.

\clearpage
\subsection{Derivation of the representation-space objective}
\label{app:reformulation}

In this section, we derive the objective in \cref{eq:latent_objective}
from the posterior KL objective in \cref{eq:recursive_objective}.
We fix the state-goal pair $(s,g)$, the number of intermediate
waypoints $K\geq 1$, and the data-collection policy $\hat{\pi}$.
Following the setup in \Cref{sec:problem_formulation}, we set $w_0=s$ and $w_{K+1}=g$ and use the factorization $q(w_{1:K}\mid s,g)=\prod_{i=1}^{K}q_i(w_i\mid s,g)$.

Assuming that the waypoint sequence is a Markov chain,
the posterior can be written as
\begin{equation}
\begin{aligned}
p_{\hat{\pi}}^{\gamma,K}(w_{1:K}\mid s,g)
=
\frac{
p_{\hat{\pi}}^{\gamma,K}(w_{1:K},g\mid s)
}{
p_{\hat{\pi}}^{\gamma,K}(g\mid s)
}
=
\frac{
\prod_{i=0}^{K}
p_{\hat{\pi}}^\gamma(w_{i+1}\mid w_{0:i})
}{
p_{\hat{\pi}}^{\gamma,K}(g\mid s)
}
=
\frac{
\prod_{i=0}^{K}
p_{\hat{\pi}}^\gamma(w_{i+1}\mid w_i)
}{
p_{\hat{\pi}}^{\gamma,K}(g\mid s)
}.
\end{aligned}
\nonumber
\end{equation}

Expanding the KL divergence in \cref{eq:recursive_objective} gives
\begin{align}
&-
D_{\mathrm{KL}}
\left(
q(\cdot\mid s,g)
\;\middle\|\;
p_{\hat{\pi}}^{\gamma,K}(\cdot\mid s,g)
\right)
\notag\\
&=
\mathbb{E}_{w_{1:K}\sim q(\cdot\mid s,g)}
\left[
\sum_{i=0}^{K}
\log p_{\hat{\pi}}^\gamma(w_{i+1}\mid w_i)
-
\sum_{i=1}^{K}
\log q_i(w_i\mid s,g)
\right]
-
\log p_{\hat{\pi}}^{\gamma,K}(g\mid s).
\label{app:eq_state_objective_from_kl}
\end{align}

Applying \Cref{as:critic} with $\pi=\hat{\pi}$, we have
\begin{equation}
\log p_{\hat{\pi}}^\gamma(w_{i+1}\mid w_i)
=
f_H(w_i,w_{i+1})
+
\log p_{\hat{\pi}}^\gamma(w_{i+1})
+
\log C.
\label{app:eq_assumption1}
\end{equation}

Substituting \cref{app:eq_assumption1} into \cref{app:eq_state_objective_from_kl} yields
\begin{align}
&-
D_{\mathrm{KL}}
\left(
q(\cdot\mid s,g)
\;\middle\|\;
p_{\hat{\pi}}^{\gamma,K}(\cdot\mid s,g)
\right)
\notag\\
&=
\mathbb{E}_{w_{1:K}\sim q(\cdot\mid s,g)}
\left[
\sum_{i=0}^{K}f_H(w_i,w_{i+1})
+
\sum_{i=1}^{K}
\log\frac{
p_{\hat{\pi}}^\gamma(w_i)
}{
q_i(w_i\mid s,g)
}
\right]
+
\mathrm{const.}
\notag\\
&=
\mathbb{E}_{w_{1:K}\sim q(\cdot\mid s,g)}
\left[
\sum_{i=0}^{K}f_H(w_i,w_{i+1})
\right]
-
\sum_{i=1}^{K}
D_{\mathrm{KL}}
\left(
q_i(\cdot\mid s,g)
\;\middle\|\;
p_{\hat{\pi}}^\gamma(\cdot)
\right)
+
\mathrm{const.}
\label{app:state_objective}
\end{align}
Here, the constant is
$\log p_{\hat{\pi}}^\gamma(g)+(K+1)\log C
-\log p_{\hat{\pi}}^{\gamma,K}(g\mid s)$.
It is independent of $q$ since $(s,g)$ and $K$ are fixed.
Removing it gives the objective in \cref{eq:waypoint-critic-objective}.

We now express the objective in terms of $z_i=\psi(w_i)$. Following the notation in \Cref{sec:closed_form_waypoint_inference}, let $q_i(z_i\mid s,g)$ denote the representation density induced by $q_i(w_i\mid s,g)$ through $\psi$, and let $p_{\hat{\pi}}^\gamma(z_i)$ denote the representation marginal density induced by $p_{\hat{\pi}}^\gamma(w_i)$ through the same mapping.
Also, the joint distribution remains factorized as
$q(z_{1:K}\mid s,g)=\prod_{i=1}^{K}q_i(z_i\mid s,g)$.

By the chain rule, the KL divergence in \Cref{app:state_objective} decomposes into a representation term and an expected conditional term over states:
\begin{align}
&D_{\mathrm{KL}}
\left(
q_i(w_i\mid s,g)
\;\middle\|\;
p_{\hat{\pi}}^\gamma(w_i)
\right)
\notag\\
&=
D_{\mathrm{KL}}
\left(
q_i(z_i\mid s,g)
\;\middle\|\;
p_{\hat{\pi}}^\gamma(z_i)
\right)
+
\mathbb{E}_{z_i\sim q_i(\cdot\mid s,g)}
\left[
D_{\mathrm{KL}}
\left(
q_i(w_i\mid z_i,s,g)
\;\middle\|\;
p_{\hat{\pi}}^\gamma(w_i\mid z_i)
\right)
\right],
\notag\\
&\ge
D_{\mathrm{KL}}
\left(
q_i(z_i\mid s,g)
\;\middle\|\;
p_{\hat{\pi}}^\gamma(z_i)
\right).
\label{app:KL_chain}
\end{align}
For any $q_i(z_i \mid s,g)$ such that $D_{\mathrm{KL}}(q_i(z_i\mid s,g)\|p^\gamma_{\hat\pi}(z_i))<\infty$, the equality in \Cref{app:KL_chain} is always achievable with $q_i(w_i\mid z_i,s,g)=p^\gamma_{\hat\pi}(w_i\mid z_i)$.
Thus, with this choice, minimizing the KL divergence over representations can lead to minimizing the KL divergence in \Cref{app:state_objective}.

Under \Cref{as:uniform}, $p_{\hat{\pi}}^\gamma(\psi)$ is uniform
on $\mathbb{S}^{d-1}$, so $\log p_{\hat{\pi}}^\gamma(z_i)$
is constant. Therefore, the KL divergence simplifies as follows:
\begin{align}
-
D_{\mathrm{KL}}
\left(
q_i(z_i\mid s,g)
\;\middle\|\;
p_{\hat{\pi}}^\gamma(z_i)
\right)
&=
-\mathbb{E}_{z_i\sim q_i(\cdot\mid s,g)}
\left[
\log q_i(z_i\mid s,g)
\right]
+
\mathrm{const.}
\notag\\
&=
\mathcal{H}(q_i(\cdot\mid s,g))
+
\mathrm{const.},
\label{app:KL_to_ent}
\end{align}
where $\mathcal{H}$ denotes entropy with respect to the surface
measure on $\mathbb{S}^{d-1}$.
Substituting $f_H(w_i,w_{i+1})=z_i^\top Rz_{i+1}$ and
\cref{app:KL_to_ent} into \cref{app:state_objective}
and dropping constants gives the objective in
\cref{eq:latent_objective}:
\begin{equation}
\mathcal{J}(q)
=
\mathbb{E}_{z_{1:K}\sim q(\cdot\mid s,g)}
\Bigg[
\psi(s)^\top Rz_1
+
\sum_{i=2}^{K} z_{i-1}^\top Rz_i
+
z_K^\top R\psi(g)
\Bigg]
+
\sum_{i=1}^{K}
\mathcal{H}(q_i(\cdot\mid s,g)).
\nonumber
\end{equation}

\clearpage
\subsection{Derivation of the closed-form solution}
\label{app:slerp}

In this section, we derive the approximate solution in
\cref{eq:closed_form_high_policy} for a fixed state-goal pair $(s,g)$
and $K\geq 1$ intermediate waypoints.
For simplicity, we suppress the dependence on $(s,g)$ and write
$q(z_{1:K})=\prod_{i=1}^{K}q_i(z_i)$,
with $z_0=\psi(s)$ and $z_{K+1}=\psi(g)$.
To control the stochasticity of the waypoint representations,
we introduce an entropy weight $\alpha>0$ into \cref{eq:latent_objective}:
\begin{equation}
\mathcal{J}(q)
=
\mathbb{E}_{z_{1:K}\sim q}
\left[\sum_{i=0}^{K}z_i^\top Rz_{i+1}\right]
+
\alpha\sum_{i=1}^{K}\mathcal{H}(q_i).
\label{eq:app_slerp_objective}
\end{equation}
As in the main text, each waypoint representation is obtained by
normalizing a Gaussian sample:
\[
\tilde z_i\sim\mathcal{N}(\mu_i,\sigma_i^2 I),
\quad
z_i=\frac{\tilde z_i}{\|\tilde z_i\|_2},
\quad
\|\mu_i\|_2=1.
\]

We use a small-variance approximation to obtain both mean directions $\mu_i$
and variances $\sigma_i^2$ in closed form.
We first determine the mean directions, then use them to optimize the variances.

\paragraph{Mean directions of waypoint representations.}
With isotropic Gaussian covariance, the entropy $\mathcal{H}(q_i)$ in
\cref{eq:app_slerp_objective} depends only on the variances
and is therefore constant when optimizing the mean directions.
Under $q(z_{1:K})=\prod_{i=1}^{K}q_i(z_i)$, each waypoint
representation $z_i$ is independent of the others, allowing us
to express the objective in terms of their individual expectations:
\begin{equation}
\mathcal{J}(q)
=
\sum_{i=0}^{K}
\mathbb{E}[z_i]^\top R\mathbb{E}[z_{i+1}]
+
\mathrm{const.},
\label{eq:app_latent_mean_objective}
\end{equation}
where the constant is independent of $\mu_1,\ldots,\mu_K$.

Since $z_i$ is normalized to unit length, we approximate small
deviations from $\mu_i$ by the perpendicular component of the
Gaussian perturbation, denoted by $\xi_i$.
Its $d-1$ orthogonal coordinates each have variance $\sigma_i^2$, so
\[
\mathbb{E}\!\left[\|\xi_i\|_2^2\right]
=
(d-1)\sigma_i^2.
\]
For a small perturbation, $\|\xi_i\|_2$ approximates the angular deviation from $\mu_i$.
Thus, the small-variance assumption $(d-1)\sigma_i^2\ll 1$ ensures that the representation remains concentrated around its mean direction $\mu_i$.

In the tangent-plane approximation, $\xi_i\perp\mu_i$ and
$\|\mu_i\|_2=1$ imply
$\|\mu_i+\xi_i\|_2=\sqrt{1+\|\xi_i\|_2^2}$.
We can therefore write
\[
\mathbb{E}[z_i]
\approx
\mathbb{E}\!\left[
\frac{\mu_i}{\sqrt{1+\|\xi_i\|_2^2}}
\right]
+
\mathbb{E}\!\left[
\frac{\xi_i}{\sqrt{1+\|\xi_i\|_2^2}}
\right].
\]
The second term vanishes by symmetry since $\xi_i$ and $-\xi_i$
have the same distribution and the denominator is unchanged
under this sign reversal.
Using $(1+x)^{-1/2}\approx 1-x/2$, we obtain
\[
\mathbb{E}[z_i]
\approx
\left(
1-\frac{1}{2}\mathbb{E}\!\left[\|\xi_i\|_2^2\right]
\right)\mu_i
=
\left(
1-\frac{d-1}{2}\sigma_i^2
\right)\mu_i.
\]
Under the small-variance assumption, we use
$\mathbb{E}[z_i]\approx\mu_i$ to optimize the mean directions
independently of the variances.
Setting $v_i=R^i\mu_i$ for $i=1,\ldots,K$, maximizing
\Cref{eq:app_latent_mean_objective} under this approximation
reduces to
\[
\max_{v_1,\ldots,v_K\in\mathbb{S}^{d-1}}
\sum_{i=0}^{K}v_i^\top v_{i+1},
\]
where $v_0=\psi(s)$ and $v_{K+1}=R^{K+1}\psi(g)$.
Let $\theta\in(0,\pi)$ denote the angle between $v_0$ and $v_{K+1}$.
The sum is maximized by placing the directions at equal angular
intervals $\Delta=\theta/(K+1)$ along the shortest geodesic
connecting them, yielding the spherical linear interpolation
\[
v_i^*
=
\operatorname{slerp}\!\left(
v_0,v_{K+1},\frac{i}{K+1}
\right).
\]
Using $\mu_i=(R^\top)^i v_i$, we obtain
\begin{equation}
\mu_i^*(s,g)
=
(R^\top)^i
\operatorname{slerp}\!\left(
\psi(s),R^{K+1}\psi(g),\frac{i}{K+1}
\right).
\label{eq:app_slerp_mean}
\end{equation}
The case $\theta=0$ follows by continuity, and we exclude $\theta=\pi$, where the shortest arc is not unique.

\paragraph{Variances of waypoint representations.}
We now optimize the variances in \cref{eq:app_slerp_objective},
with the mean directions fixed at $\mu_i^*$ from
\cref{eq:app_slerp_mean}.
For each $i$, we rewrite the objective to isolate its dependence on $\sigma_i$:
\begin{equation}
\mathcal{J}(q)
=
\left(
R^\top\mathbb{E}[z_{i-1}]
+
R\,\mathbb{E}[z_{i+1}]
\right)^\top
\mathbb{E}[z_i]
+
\alpha\mathcal{H}(q_i)
+
\mathrm{const.},
\label{eq:app_latent_variance_objective}
\end{equation}
where the constant is independent of $\sigma_i$.
To optimize the variances, we retain the reduction in mean
length derived above.
A larger variance allows sampled waypoint sequences to cover
a broader region of representation space, increasing entropy
while reducing the expected inner products between neighboring
representations.
Using only $\mathbb{E}[z_i]\approx\mu_i^*$ makes the first
term independent of the variances and fails to capture this tradeoff.

Set $\mu_0^*=\psi(s)$, $\mu_{K+1}^*=\psi(g)$, and
$\sigma_0^2=\sigma_{K+1}^2=0$ for the fixed endpoints.
Substituting
$\mathbb{E}[z_j]\approx
(1-\frac{d-1}{2}\sigma_j^2)\mu_j^*$
for $j=i,i+1$ into each neighboring inner product gives
\[
\mathbb{E}[z_i]^\top R\,\mathbb{E}[z_{i+1}]
\approx
\left[
1-\frac{d-1}{2}
\left(\sigma_i^2+\sigma_{i+1}^2\right)
\right]
(\mu_i^*)^\top R\mu_{i+1}^*,
\]
where we omit the products $\sigma_i^2\sigma_{i+1}^2$
under the small-variance assumption.
Applying this expansion to the two pairs $(i-1,i)$ and
$(i,i+1)$ in \cref{eq:app_latent_variance_objective},
the terms proportional to $\sigma_i^2$ are
\[
-\frac{d-1}{2}\sigma_i^2
\left[
(\mu_{i-1}^*)^\top R\mu_i^*
+
(\mu_i^*)^\top R\mu_{i+1}^*
\right]
=
-(d-1)\sigma_i^2\cos\Delta.
\]
Recall that consecutive transformed directions $v_j^*=R^j\mu_j^*$
are separated by the angle $\Delta=\theta/(K+1)$.
Thus, each inner product satisfies
\[
(\mu_j^*)^\top R\mu_{j+1}^*
=
(v_j^*)^\top v_{j+1}^*
=
\cos\Delta,
\qquad j=i-1,i.
\]
Collecting the remaining terms independent of $\sigma_i$
into a constant, we can therefore write
\begin{equation}
\mathcal{J}(q)
\approx
-(d-1)\sigma_i^2\cos\Delta
+
\alpha\mathcal{H}(q_i)
+
\mathrm{const.}
\label{eq:app_variance_reduced_objective}
\end{equation}

Under the same small-variance assumption, we approximate
$\mathcal{H}(q_i)$ by the entropy of an isotropic Gaussian
in the $(d-1)$-dimensional tangent plane.
Each of its $d-1$ independent coordinates follows
$X\sim\mathcal{N}(0,\sigma_i^2)$, with density
\[
p_X(x)
=
\frac{1}{\sqrt{2\pi\sigma_i^2}}
\exp\!\left(-\frac{x^2}{2\sigma_i^2}\right).
\]
Substituting this density into $-\mathbb{E}[\log p_X(X)]$
gives
\[
-\mathbb{E}[\log p_X(X)]
=
\frac{1}{2}\log(2\pi\sigma_i^2)
+
\frac{\mathbb{E}[X^2]}{2\sigma_i^2}
=
\frac{1}{2}\log(2\pi e\,\sigma_i^2),
\]
where $\mathbb{E}[X^2]=\sigma_i^2$.
Since these coordinates are independent, their entropies add,
giving the following expression under the tangent-plane approximation:
\[
\mathcal{H}(q_i)
\approx
\frac{d-1}{2}\log(2\pi e\,\sigma_i^2)
=
\frac{d-1}{2}\bigl(1+\log(2\pi)\bigr)
+
(d-1)\log\sigma_i.
\]
Substituting the entropy expression above into
\cref{eq:app_variance_reduced_objective},
and collecting terms independent of $\sigma_i$ into a constant,
we obtain
\begin{equation}
\mathcal{J}(q)
\approx
-(d-1)\sigma_i^2\cos\Delta
+
\alpha(d-1)\log\sigma_i
+
\mathrm{const.}
\label{eq:app_approx_objective_latent}
\end{equation}
Differentiating the right-hand side with respect to $\sigma_i$
and setting the derivative to zero gives
\[
-2(d-1)\cos\Delta\,\sigma_i
+
\frac{\alpha(d-1)}{\sigma_i}
=
0.
\]
Using $\Delta=\theta/(K+1)$, we obtain
\begin{equation}
\sigma_i^{2*}(s,g)
=
\frac{\alpha}{2\cos\Delta}
=
\frac{\alpha}{
2\cos\!\left(\theta/(K+1)\right)
},
\qquad i=1,\ldots,K.
\label{eq:app_slerp_variance}
\end{equation}
For $d\geq 2$ and $\alpha>0$, the approximate objective in
\cref{eq:app_approx_objective_latent} is strictly concave in
$\sigma_i>0$ since $\cos\Delta>0$, so the positive stationary
point is its unique maximizer.
Together, \Cref{eq:app_slerp_mean,eq:app_slerp_variance} yield
the approximate solution in \cref{eq:closed_form_high_policy}
under the small-variance condition
$\alpha(d-1)/(2\cos\Delta)\ll 1$.

\clearpage
\subsection{Derivation of the distribution over the number of waypoints}
\label{app:k-selection}

In this section, we derive a distribution over the number of intermediate waypoints for a given state-goal pair.
Following the setup in \Cref{sec:problem_formulation}, we use the notation
\[
w_{1:K}=(w_1,\ldots,w_K),
\quad
w_0=s,
\quad
w_{K+1}=g.
\]
The $K$ intermediate waypoints divide the trajectory into $K+1$ segments, where the temporal intervals are sampled independently from a geometric distribution.
We define $p_{\hat{\pi}}^{\gamma,K}(g\mid s)$ as the resulting distribution of $g$ under $\hat{\pi}$ given the starting state $s$.
For $K=0$, this distribution reduces to $p_{\hat{\pi}}^\gamma(g\mid s)$, with a time interval sampled from a single geometric distribution and no intermediate waypoints.

Recall that $f_H$ is parameterized as
\begin{equation}
f_H(s,s^+)
=
\psi(s)^\top R\psi(s^+)
=
\left(R^\top\psi(s)\right)^\top\psi(s^+),
\label{eq:high_level_critic_rotation}
\nonumber
\end{equation}
where $\psi(s)\in\mathbb S^{d-1}$ and $R$ is orthogonal. Next, we relate $R^\top\psi(s)$ to the expected representation of a future state sampled from $p_{\hat{\pi}}^\gamma(\cdot\mid s)$.

\begin{lemma}[Expected positive alignment]
\label{lem:high_level_positive_alignment}
Assume that, at convergence, the contrastive representations are uniformly distributed on $\mathbb S^{d-1}$, with $d\gg1$.
The expected representation of a future state after a time interval sampled from a geometric distribution satisfies
\begin{equation}
\mathbb E_{s^+\sim p_{\hat{\pi}}^\gamma(\cdot\mid s)}
\left[\psi(s^+)\right]
\approx
\rho R^\top\psi(s),
\nonumber
\label{eq:high_level_positive_alignment}
\end{equation}
where $\rho\in(0,1]$ is a coefficient shared across states.
\end{lemma}
\begin{proof}
The result follows directly from the proof of Lemma~2 in Appendix~B.2.2 of \citet{bastankhah2026demystifying}, with $\phi(s,a)$ replaced by $R^\top\psi(s)$.
\end{proof}
\Cref{lem:high_level_positive_alignment} identifies $R^\top\psi(s)$ as the predicted direction after one geometric transition, while $\rho$ captures the contraction in the expected representation.
The following result extends this relation to $K+1$ geometric transitions.
\\

\begin{lemma}[Recursive composition of geometric transitions]
\label{lem:recursive-composition}
Let $s_0=s$ and recursively sample $s_i\sim p_{\hat{\pi}}^\gamma(\cdot\mid s_{i-1})$ for $i=1,\ldots,K+1$. Then the expected representation of $s_{K+1}$ satisfies
\begin{equation}
\mathbb E_{s_{K+1}\sim p_{\hat{\pi}}^{\gamma,K}(\cdot\mid s)}
\left[\psi(s_{K+1})\right]
\approx
\rho^{K+1}(R^\top)^{K+1}\psi(s),
\label{eq:negative_binomial_representation}
\nonumber
\end{equation}
where $\rho$ is the coefficient in
\Cref{lem:high_level_positive_alignment}.
\end{lemma}

\begin{proof}
By definition, $s_{K+1}\sim p_{\hat{\pi}}^{\gamma,K}(\cdot\mid s)$.
For the expected representation, the case $K=0$ follows directly from
\Cref{lem:high_level_positive_alignment}.
Suppose that
$\mathbb E[\psi(s_K)\mid s_0=s]
\approx\rho^K(R^\top)^K\psi(s)$.
Assuming that the process is Markov with transition kernel $p_{\hat{\pi}}^\gamma$,
the law of iterated expectations gives
\begin{align}
\mathbb E[\psi(s_{K+1})\mid s_0=s]
&=
\mathbb E\left[
    \mathbb E[\psi(s_{K+1})\mid s_K]
    \,\middle|\,
    s_0=s
\right]
\nonumber\\
&\approx
\mathbb E\left[
    \rho R^\top\psi(s_K)
    \,\middle|\,
    s_0=s
\right]
\nonumber\\
&=
\rho R^\top
\mathbb E[\psi(s_K)\mid s_0=s]
\nonumber\\
&\approx
\rho R^\top
\rho^K(R^\top)^K\psi(s)
\nonumber\\
&=
\rho^{K+1}(R^\top)^{K+1}\psi(s).
\nonumber
\end{align}
The stated approximation follows by induction.
\end{proof}

By \Cref{lem:recursive-composition}, the expected representation of
$s_{K+1}$ is approximated by $\rho^{K+1}(R^\top)^{K+1}\psi(s)$.
Since $\psi(s)$ has unit norm and $R$ is orthogonal, this vector has
norm $\rho^{K+1}$.
Normalizing it therefore gives the predicted direction
\[
\mu_K(s)\triangleq(R^\top)^{K+1}\psi(s).
\]
We define its alignment with the goal representation as
\begin{align}
\ell_K(s,g)
&\triangleq
\mu_K(s)^\top\psi(g)
\notag\\
&=
\left((R^\top)^{K+1}\psi(s)\right)^\top\psi(g)
\notag\\
&=
\psi(s)^\top R^{K+1}\psi(g).
\label{eq:depth_compatibility}
\nonumber
\end{align}
This recovers the score in \cref{eq:k_selection_score}, which measures
how well the predicted direction for $K$ intermediate waypoints aligns
with the goal representation.

We now convert this alignment score into a distribution over $K$.
Since $\mu_K(s)\in\mathbb S^{d-1}$, we model the goal representation
with a von Mises--Fisher likelihood centered at $\mu_K(s)$:
\begin{align}
p(\psi(g)\mid\psi(s),K)
&\triangleq
C_d(\kappa)\exp\!\left(
\kappa\mu_K(s)^\top\psi(g)
\right)
\notag\\
&=
C_d(\kappa)\exp\!\left(
\kappa\ell_K(s,g)
\right),
\nonumber
\end{align}
as in \cref{eq:k_selection_likelihood}.
Here, $\kappa>0$ is a fixed concentration parameter.
The normalizing constant $C_d(\kappa)$ depends only on $d$ and $\kappa$
and is therefore independent of both $s$ and $K$.

For a prior $p(K)$ over $\{0,\ldots,K_{\max}\}$, Bayes' rule gives
\begin{align}
p(K\mid\psi(s),\psi(g))
&=
\frac{
p(\psi(g)\mid\psi(s),K)p(K)
}{
\sum_{J=0}^{K_{\max}}
p(\psi(g)\mid\psi(s),J)p(J)
}
\notag\\
&=
\frac{
\exp\!\left(\kappa\ell_K(s,g)\right)p(K)
}{
\sum_{J=0}^{K_{\max}}
\exp\!\left(\kappa\ell_J(s,g)\right)p(J)
},
\nonumber
\end{align}
where the shared normalizing constant $C_d(\kappa)$ cancels.

With the uniform prior $p(K)=1/(K_{\max}+1)$, this reduces to
\begin{equation}
p(K\mid\psi(s),\psi(g))
=
\operatorname{Softmax}_{K\in\{0,\ldots,K_{\max}\}}
\left(
\frac{\ell_K(s,g)}{\eta}
\right),
\label{eq:uniform_recursion_depth_posterior}
\nonumber
\end{equation}
where $\eta=1/\kappa$, yielding \cref{eq:k_selection_distribution}.

\clearpage
\section{Additional Experimental Results}
\label{app:add_exp}
\subsection{Ablation Study}
\begin{figure}[H]
    \centering
    \includegraphics[width=1.0\linewidth]{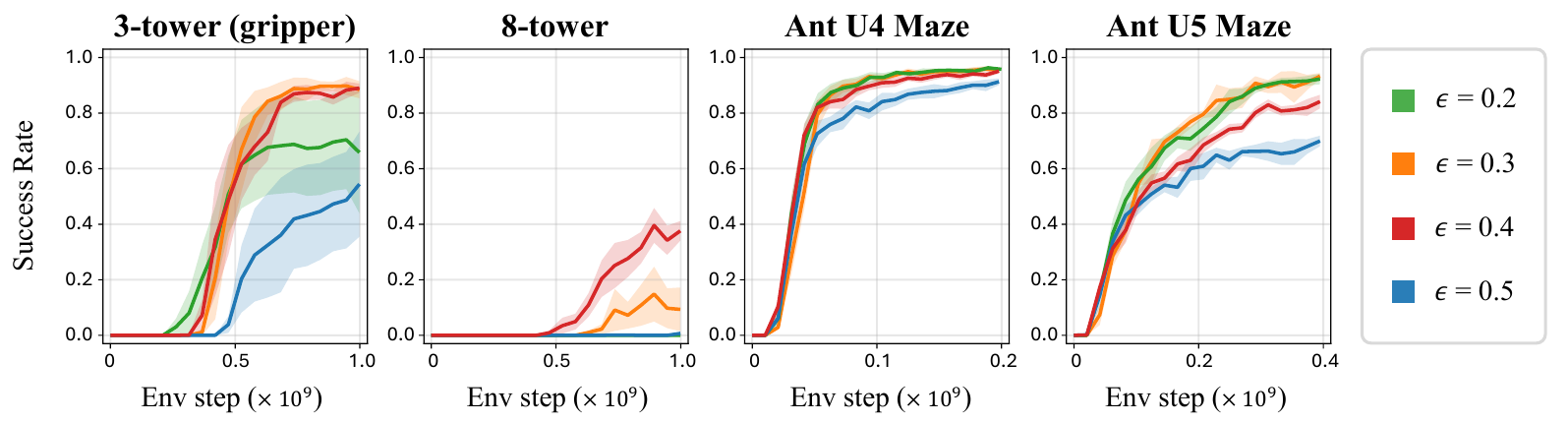}
    \vspace{-10pt}
    \caption{\textbf{Ablation on the threshold $\epsilon$.} }
    \vspace{-25pt}
    \label{fig:epsilon_ablation}
\end{figure}
\paragraph{Ablation on the threshold $\epsilon$.}

We investigate the effect of the threshold $\epsilon$ used to select the waypoint index~$i^*$.
As described in \Cref{eq:waypoint_index_selection}, a larger $\epsilon$ favors earlier waypoints, leading to more conservative subgoal selection.
Figure~\ref{fig:epsilon_ablation} shows that C-Slerp performs well across multiple threshold choices, although sensitivity varies across tasks.
Performance is relatively robust across $\epsilon \in \{0.2, 0.3, 0.4\}$ in all evaluated environments except the 8-tower task.
The 8-tower task is more sensitive to the threshold, with $\epsilon=0.4$ achieving the highest success rate, whereas $\epsilon=0.2$ and $0.3$ perform better on Ant U5 Maze.
We use $\epsilon=0.4$ as the default configuration for our algorithm across all tasks.
\\
\\
\begin{figure}[H]
    \centering
    \includegraphics[width=0.65\linewidth]{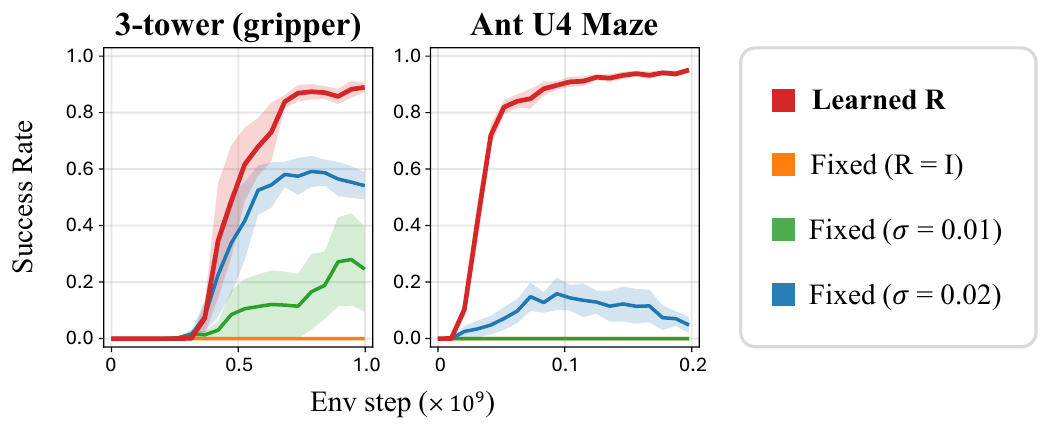}
    \vspace{-0pt}
    \caption{\textbf{Ablation on learning the rotation matrix $R$.}}
    \vspace{-25pt}
    \label{fig:rotation_ablation}
\end{figure}
\paragraph{Ablation on learning the rotation matrix $R$.}
We investigate the role of learning the rotation matrix $R$ in $f_H(s,s^+)=\psi(s)^\top R\psi(s^+)$ by comparing it with fixed alternatives.
As described in \Cref{sec:planning_representations}, we parameterize $R$ through the Cayley transform, where $\Omega=A-A^\top$ and $R=(I+\Omega)^{-1}(I-\Omega)$.
Using this parameterization, we compare our default approach of learning $A$ with three fixed variants.
The first sets $A=0$, yielding the identity matrix $R=I$.
The other two independently sample $A_{ij}\sim\mathcal{N}(0,\sigma^2)$ with $\sigma\in\{0.01,0.02\}$.
In all three variants, we stop gradients to $A$, keeping $R$ fixed throughout training.
As shown in Figure~\ref{fig:rotation_ablation}, setting $R=I$ makes $f_H$ \textit{symmetric} and prevents successful learning in the evaluated tasks.
In contrast, the two variants with randomly initialized, fixed $R$ introduce asymmetry and enable learning, but yield unstable performance.
These results suggest that asymmetry alone is insufficient for stable performance, highlighting the benefit of \textit{learning} $R$ to capture temporal relationships between states.
\\
\\
\\

\begin{figure}[H]
    \centering
    \includegraphics[width=0.7\linewidth]{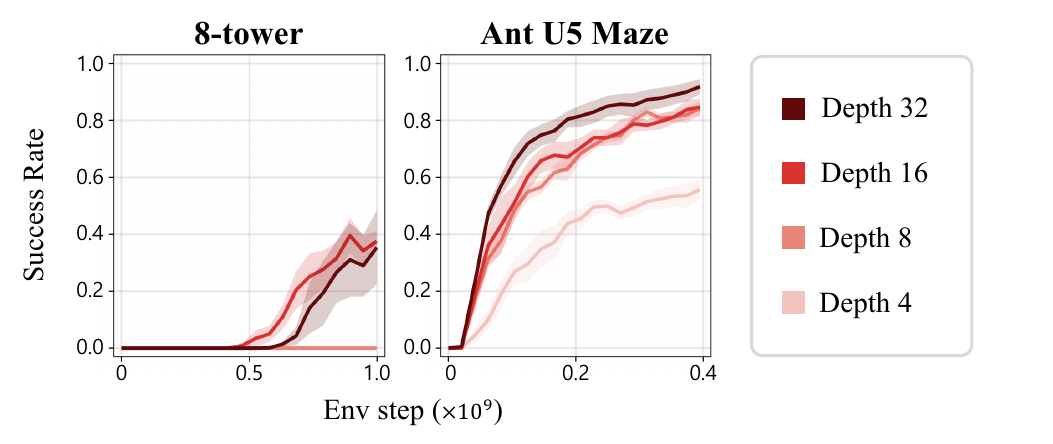}
    \vspace{-5pt}
    \caption{\textbf{Ablation on network depth.}}
    \vspace{-5pt}
    \label{fig:depth_comparison}
\end{figure}
\paragraph{Ablation on network depth.}
\citet{wang20261000} report that increasing network depth improves the goal-reaching performance of CRL.
Since C-Slerp also uses contrastive representations, we examine whether it benefits similarly from increased network depth.
The results show that deeper networks generally achieve higher success rates.
On Ant U5 Maze, performance improves as depth increases, while on the 8-tower task, depths of 16 and 32 perform comparably and outperform shallower networks, suggesting that sufficient network depth is needed for this task.
\\
\\
\begin{figure}[H]
    \centering
    \includegraphics[width=0.7\linewidth]{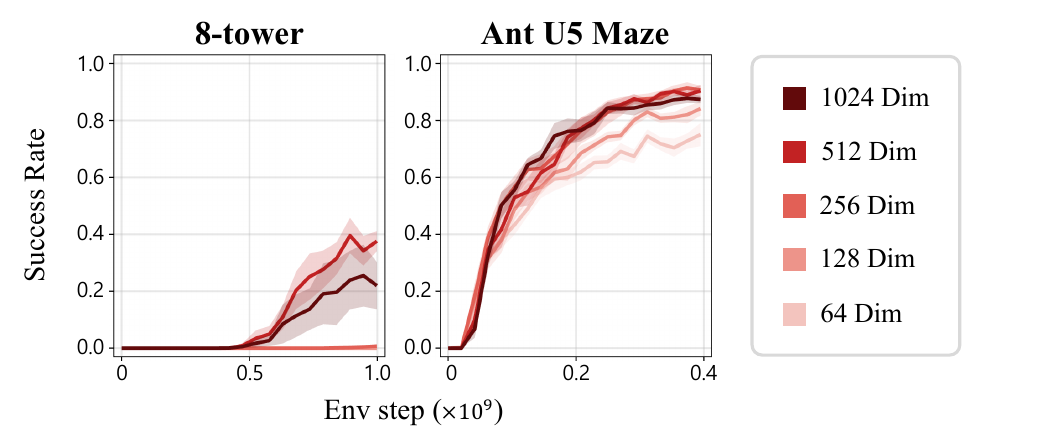}
    \vspace{-5pt}
    \caption{\textbf{Ablation on representation dimension $d$.}}
    \vspace{-5pt}
    \label{fig:repr_comparison}
\end{figure}
\paragraph{Ablation on representation dimension.}
We investigate the effect of the representation dimension $d$.
Since the learned rotation matrix $R$ in $f_H(s,s^+)=\psi(s)^\top R\psi(s^+)$ is a \textit{linear} operator, we examine whether increasing the representation dimension $d$ provides greater expressive capacity and improves performance.
The results show that performance is relatively robust once the representation dimension is sufficiently large, with further increases providing no consistent benefit.
In particular, on the 8-tower task, only dimensions of $d \geq 256$ achieve nonzero success among the tested values.
We hypothesize that this limited benefit from further increasing $d$ arises because the encoder $\psi$ already provides sufficient \textit{nonlinear} expressivity for the linear operator $R$ to model temporal relationships effectively.

\clearpage
\subsection{CRL with larger discount factors}

\begin{figure}[H]
    \centering
    \includegraphics[width=0.7\linewidth]{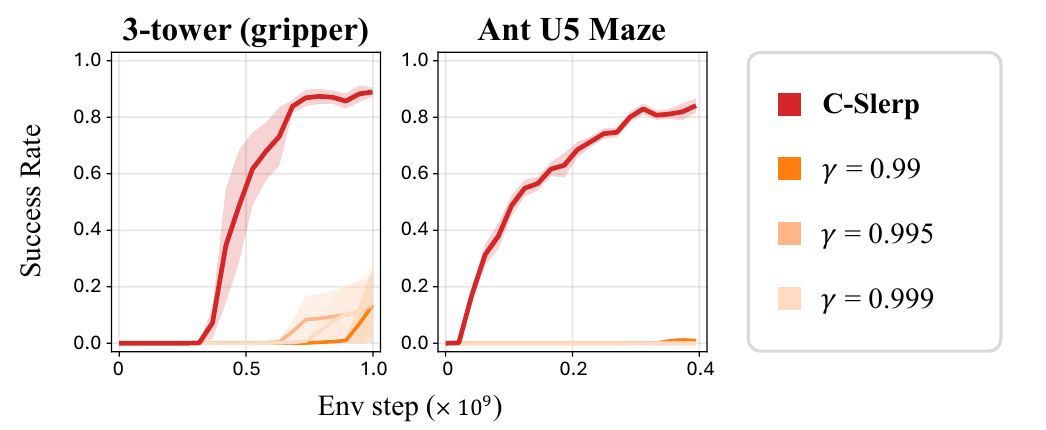}
    \vspace{-5pt}
    \caption{\textbf{Comparison with CRL with larger discount factors $\gamma$.}}
    \vspace{-5pt}
    \label{fig:crl_gamma_ablation}
\end{figure}

\paragraph{Comparison with CRL using larger discount factors $\gamma$.}
As discussed in \Cref{sec:preliminaries}, CRL samples future states for contrastive learning using a temporal offset drawn from $\operatorname{\textsc{Geom}}(1-\gamma)$.
Thus, a \textit{natural} approach for adapting CRL to long-horizon tasks is to \textit{increase} the discount factor $\gamma$.

We investigate this approach by evaluating CRL across $\gamma \in \{0.99, 0.995, 0.999\}$ on long-horizon tasks within our benchmark, while C-Slerp uses $\gamma=0.99$.
As shown in Figure~\ref{fig:crl_gamma_ablation}, CRL achieves low success rates across the evaluated tasks, and increasing $\gamma$ does not close the performance gap with C-Slerp.
One possible explanation is that a larger $\gamma$ places more weight on temporally \textit{distant} future states.
This may make it harder for the network to distinguish temporally close and distant goals, limiting the effectiveness of the learned representations.
These results suggest that simply extending the temporal horizon of contrastive learning does not offer the same benefits as planning with an appropriate number of waypoints.

\clearpage
\subsection{Visualizations}
\paragraph{Evolution of $K$ and $i^*$ over time.}
Figures~\ref{fig:KI_selection} and~\ref{fig:KI_selection_2} show the mean and sampled values of $K+1$, together with the selected waypoint index $i^*$, over a single episode in each environment.
Since $K$ denotes the number of intermediate waypoints, $K+1$ represents the total number of transition segments from the current state to the goal.
We plot $K+1$ to align with the waypoint indexing, where the goal has index $K+1$, allowing direct comparison with the selected index $i^*$.

In general, the mean value of $K+1$ decreases as the agent approaches the goal, suggesting that the learned representation captures the remaining temporal distance to the goal.
Notably, in Humanoid U Maze, the selected $K+1$ values increase again after the humanoid reaches the goal, as it subsequently falls and cannot recover, causing its torso height to deviate from the target.

\paragraph{Visualizations of environment rollouts with corresponding representations.}
Figures~\ref{fig:rollout_1} and \ref{fig:rollout_2} visualize the environment rollouts together with spherical PCA~\citep{liu2019spherical} projections of the corresponding latent representations.
Across the evaluated environments, these visualizations illustrate how C-Slerp adjusts the number of waypoints and selects a latent subgoal for the policy to pursue as it progresses toward the goal.

\paragraph{Visualization of waypoint representations in maze environments.}
Figure~\ref{fig:waypoint_energy} visualizes the inferred waypoints $z_{1:K}$ and the scores $f_L$ over the maze.
We first encode all $(x,y)$ positions on a grid with $0.2$\,m spacing using the state encoder $\psi$.
For each inferred waypoint $z_i$, we retrieve its nearest neighbor among these representations and plot the corresponding position in the maze.
We also evaluate $f_L(s,a,s')=\phi(s,a)^\top\psi(s')$ at every grid point and display the resulting scores as a heatmap, providing a proxy for how easily the current policy can reach each position.
The visualizations indicate that, although the waypoints are generated entirely in the latent space, they correspond to intermediate $(x,y)$ positions along the path toward the goal.

\clearpage

\begin{figure}[t]
    \centering
    \includegraphics[width=0.95\linewidth]{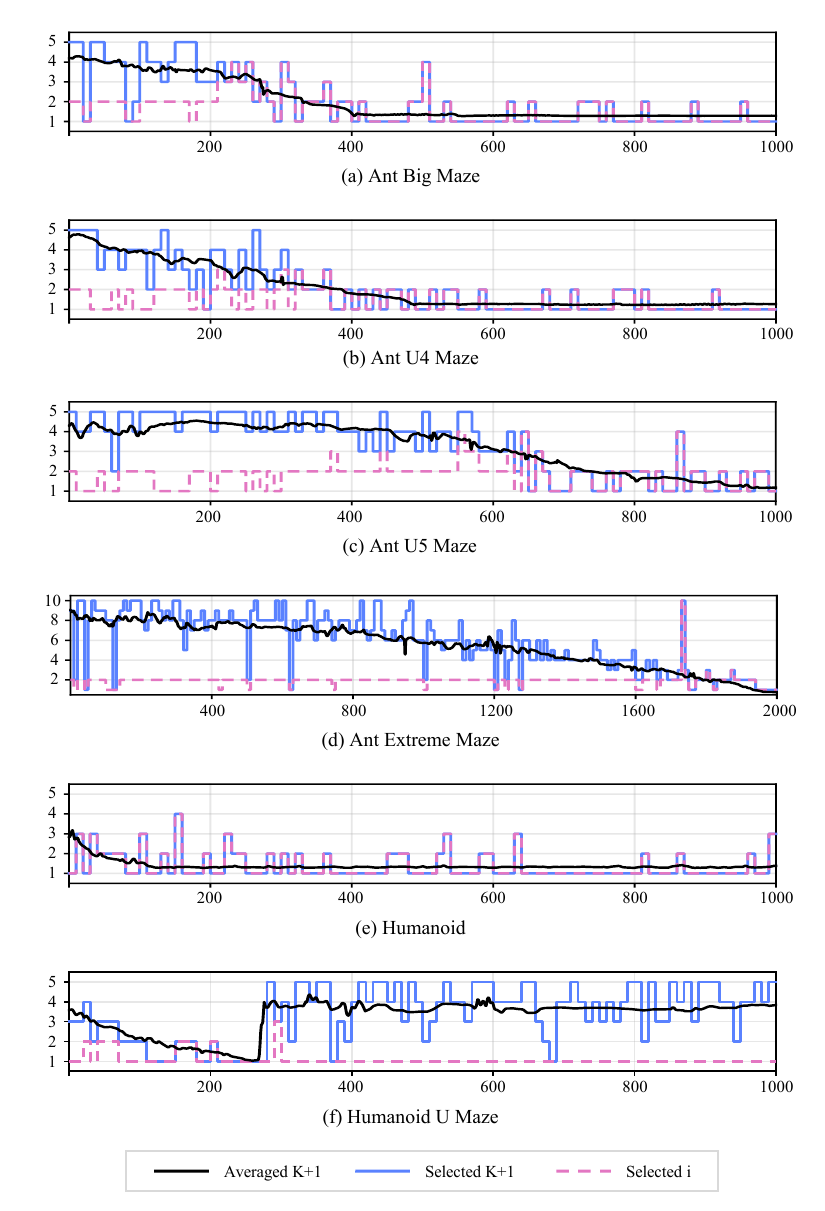}
    \vspace{-5pt}
    \caption{\textbf{Evolution of $K$ and $i^*$ over time in JaxGCRL environments.}
    We plot the mean and sampled values of $K+1$, together with the selected waypoint index $i^*$, against environment steps.
    As the agent approaches the goal, the predicted number of waypoints decreases.}
    \vspace{-5pt}
    \label{fig:KI_selection}
\end{figure}

\clearpage
\begin{figure}[t]
    \centering
    \includegraphics[width=0.95\linewidth]{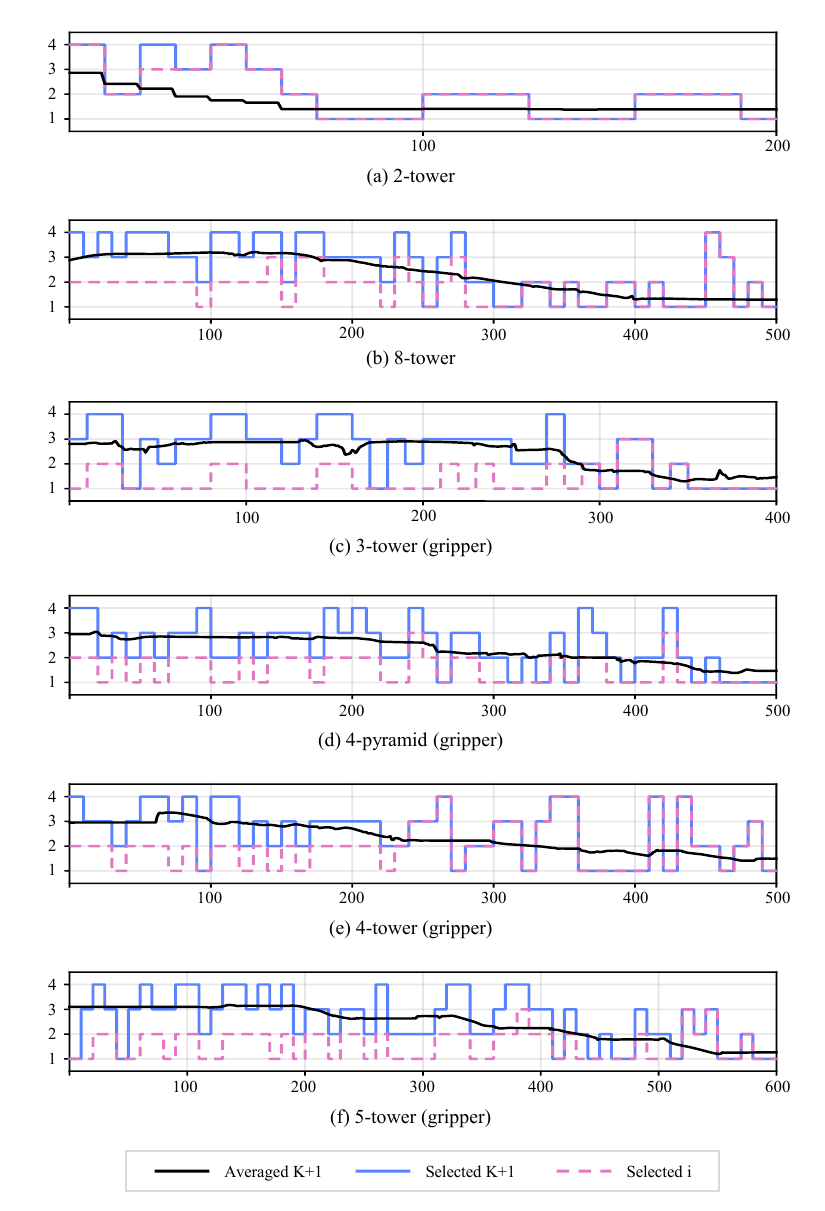}
    \vspace{-5pt}
    \caption{\textbf{Evolution of $K$ and $i^*$ over time in BuilderBench environments.}
    We plot the mean and sampled values of $K+1$, together with the selected waypoint index $i^*$, against environment steps.
    As the agent approaches the goal, the predicted number of waypoints decreases.}
    \vspace{-5pt}
    \label{fig:KI_selection_2}
\end{figure}

\clearpage
\begin{figure}[t]
    \centering
    \includegraphics[width=1.0\linewidth]{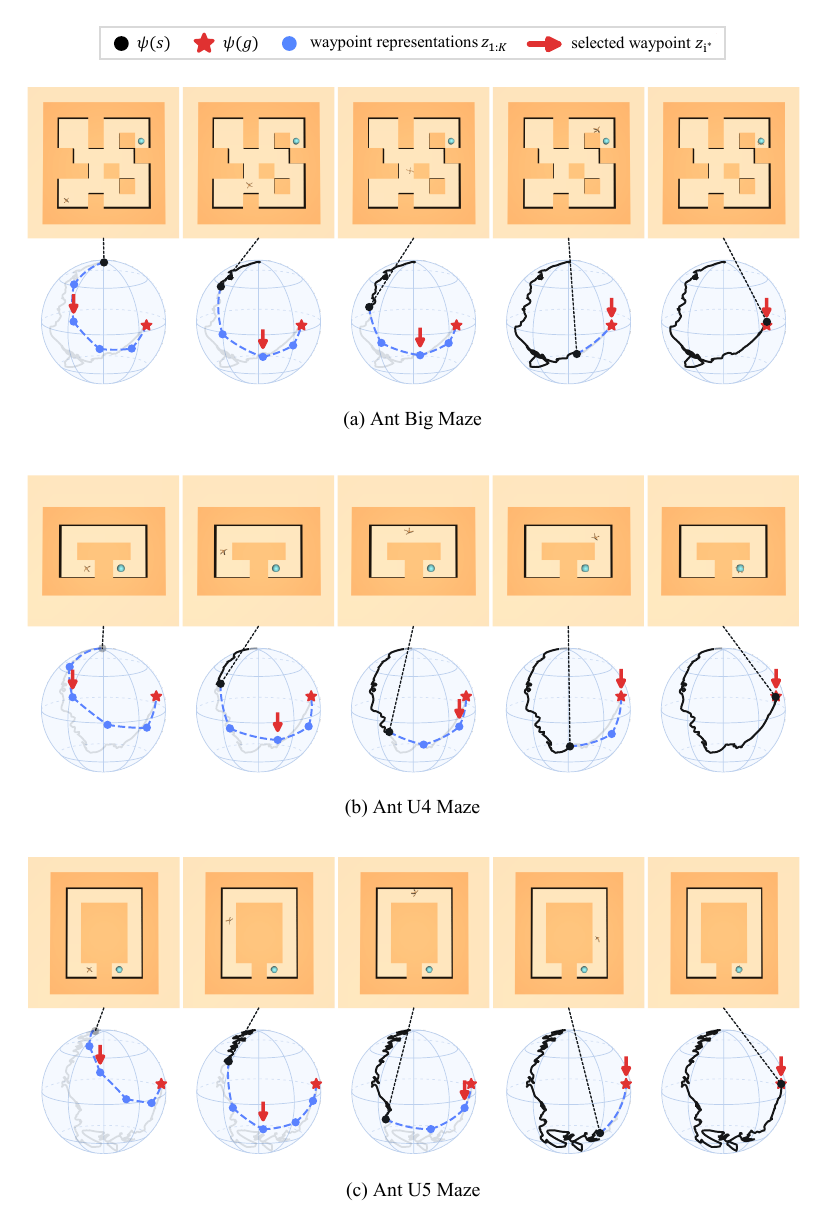}
\end{figure}

\clearpage
\begin{figure}[t]
    \centering
    \includegraphics[width=1.0\linewidth]{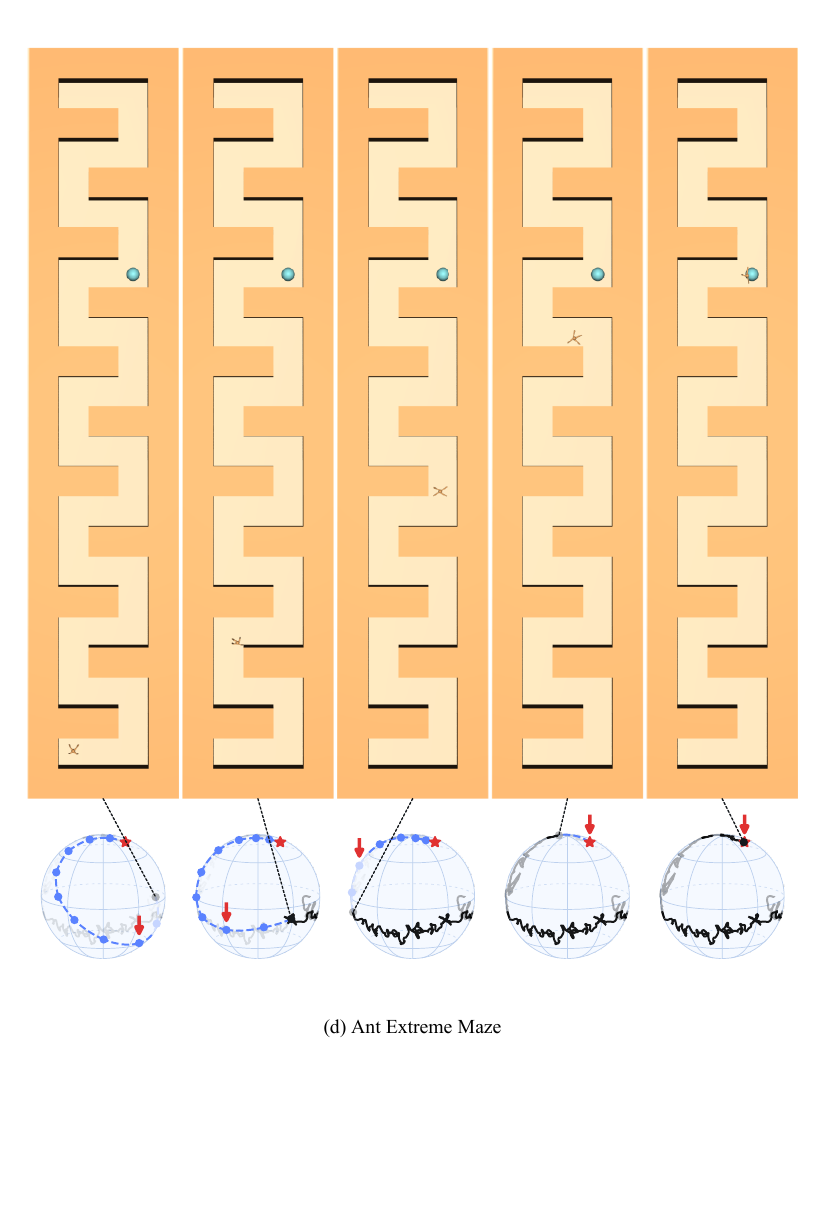}
\end{figure}

\clearpage
\begin{figure}[t]
    \centering
    \vspace{-60pt}
    \includegraphics[width=1.0\linewidth]{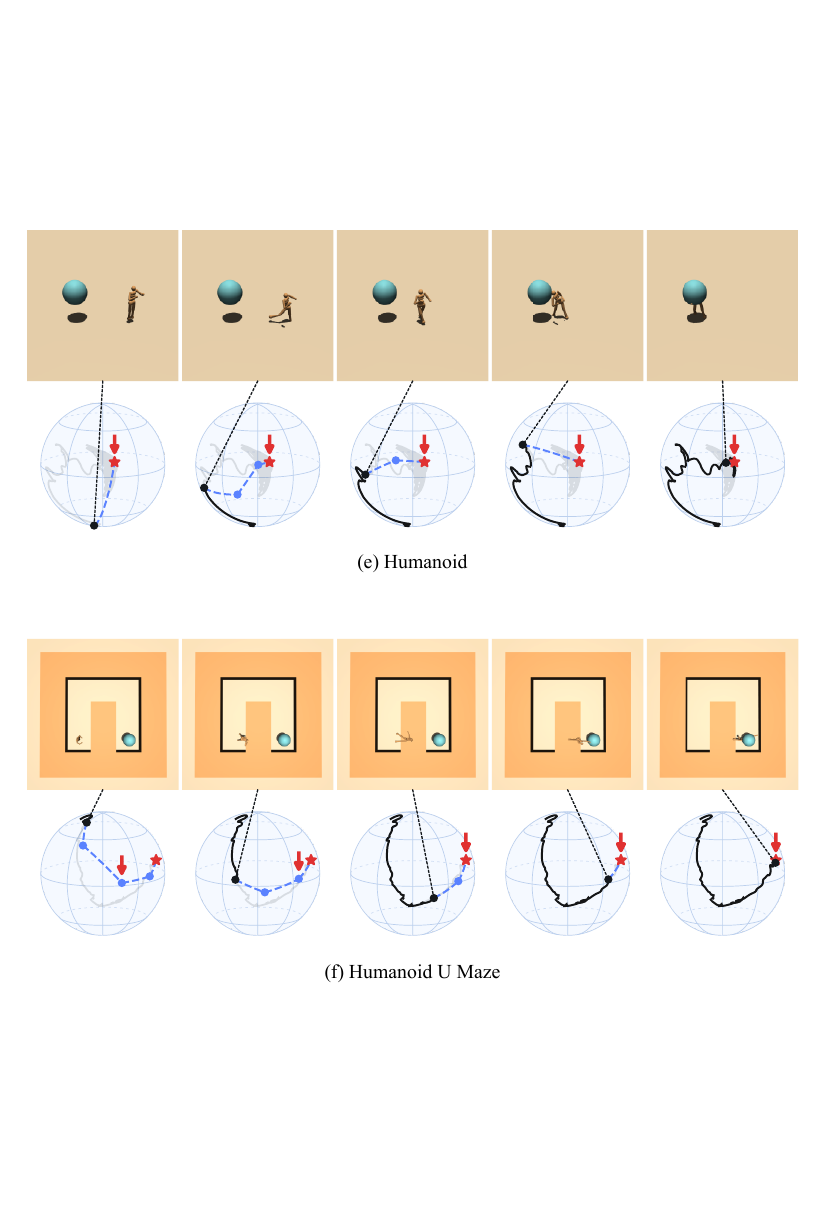}
    \vspace{-100pt}
    \caption{\textbf{Visualization of C-Slerp rollouts and corresponding latent representations in JaxGCRL.}
    For each environment, the top row shows successive states, and the bottom row shows spherical PCA~\citep{liu2019spherical} projections of the corresponding state representations and inferred waypoints.}
    \vspace{-5pt}
    \label{fig:rollout_1}
\end{figure}

\clearpage
\begin{figure}[t]
    \centering
    \includegraphics[width=1.0\linewidth]{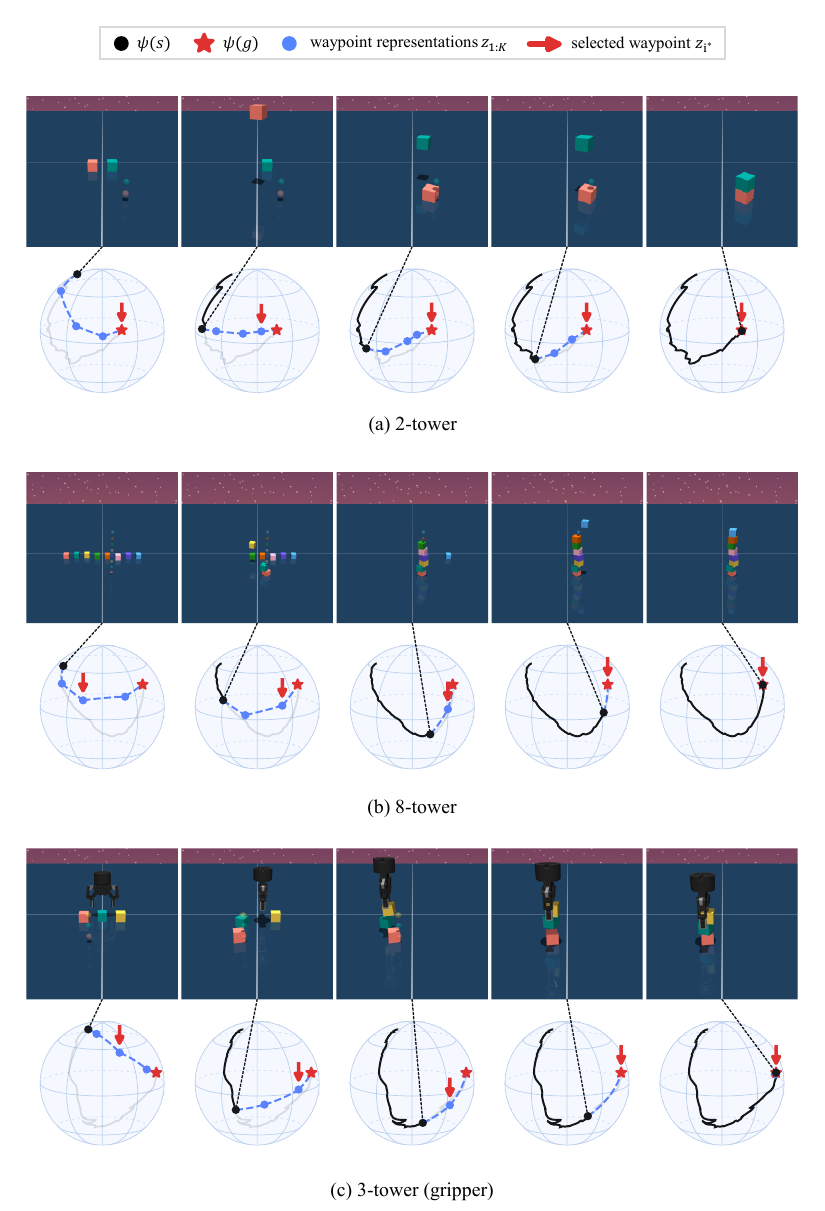}
\end{figure}

\clearpage
\begin{figure}[t]
    \centering
    \includegraphics[width=1.0\linewidth]{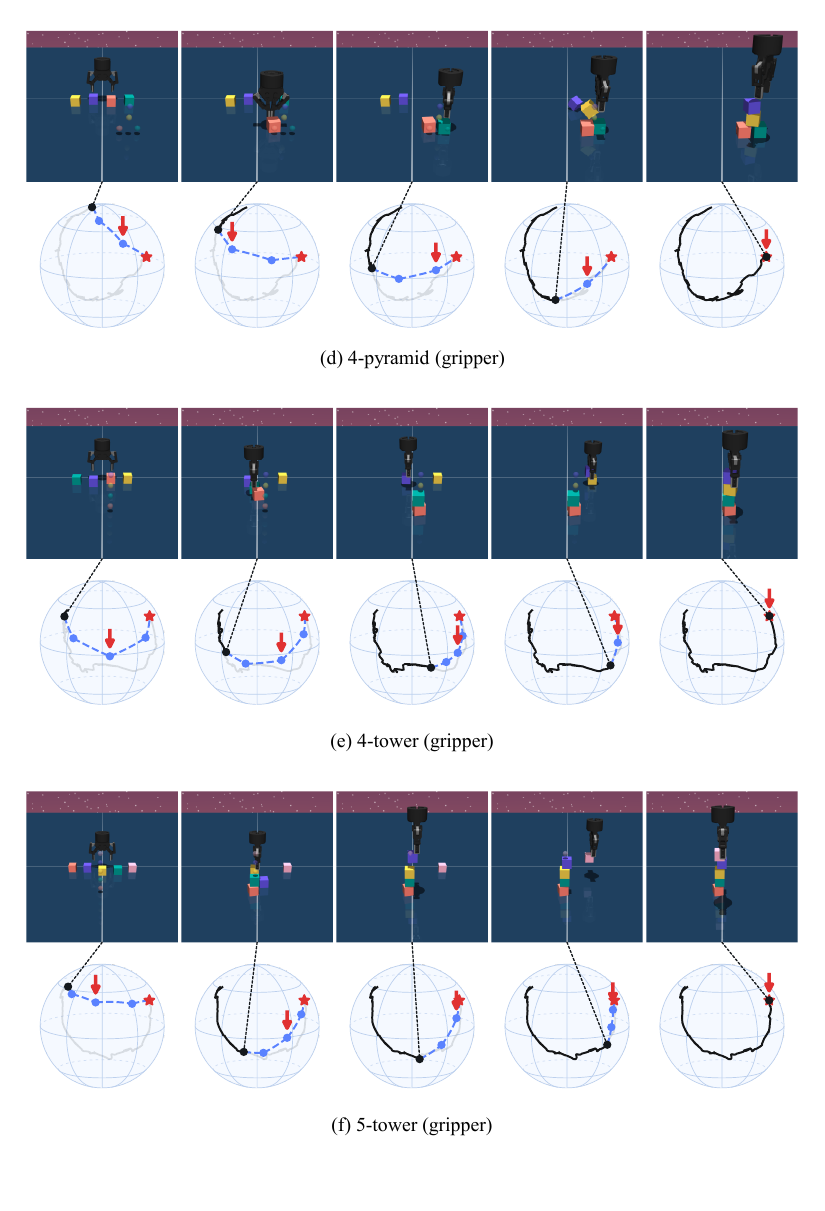}
    \vspace{-25pt}
    \caption{\textbf{Visualization of C-Slerp rollouts and corresponding latent representations in BuilderBench.}
    For each environment, the top row shows successive states, and the bottom row shows spherical PCA~\citep{liu2019spherical} projections of the corresponding state representations and inferred waypoints.}
    \vspace{-5pt}
    \label{fig:rollout_2}
\end{figure}

\clearpage
\begin{figure}[t]
    \centering
    \includegraphics[width=1.0\linewidth]{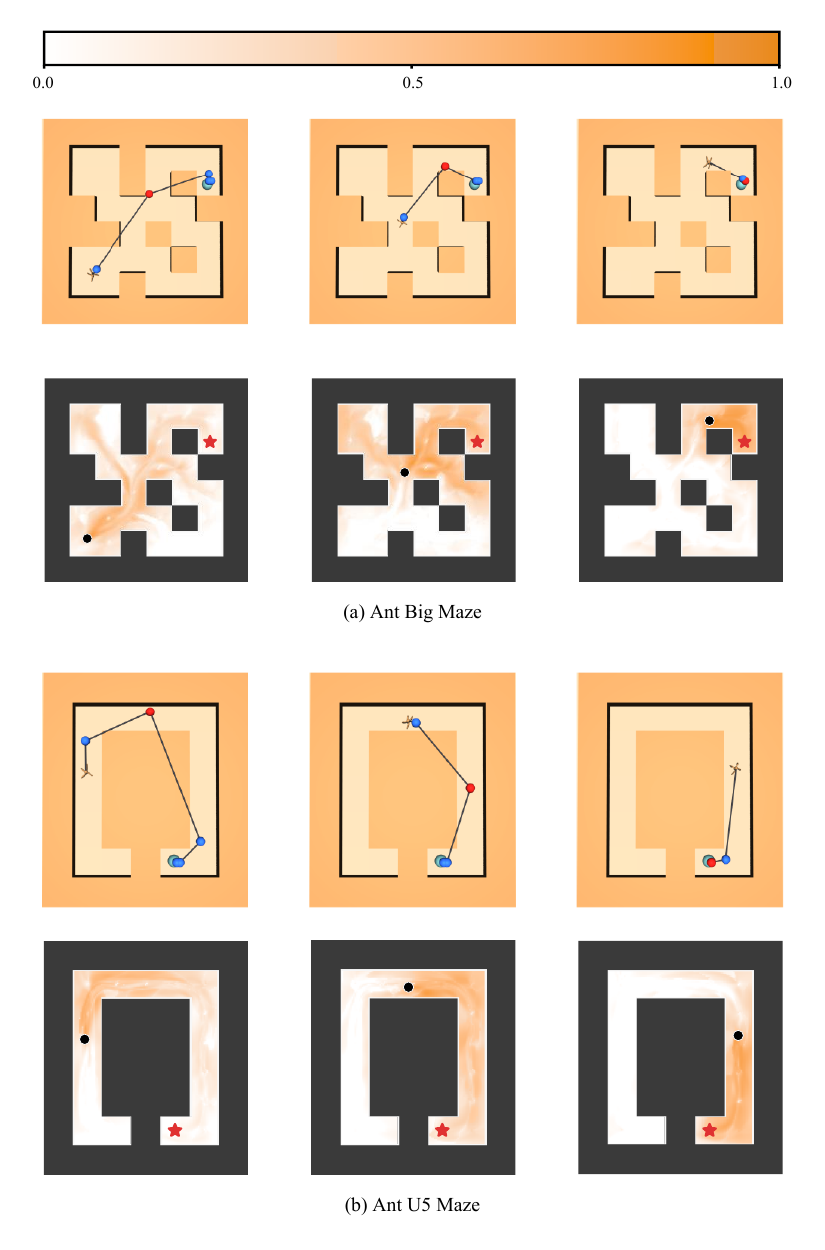}
    \vspace{-15pt}
\end{figure}

\clearpage
\begin{figure}[t]
    \centering
    \vspace{-25pt}
    \includegraphics[width=1.0\linewidth]{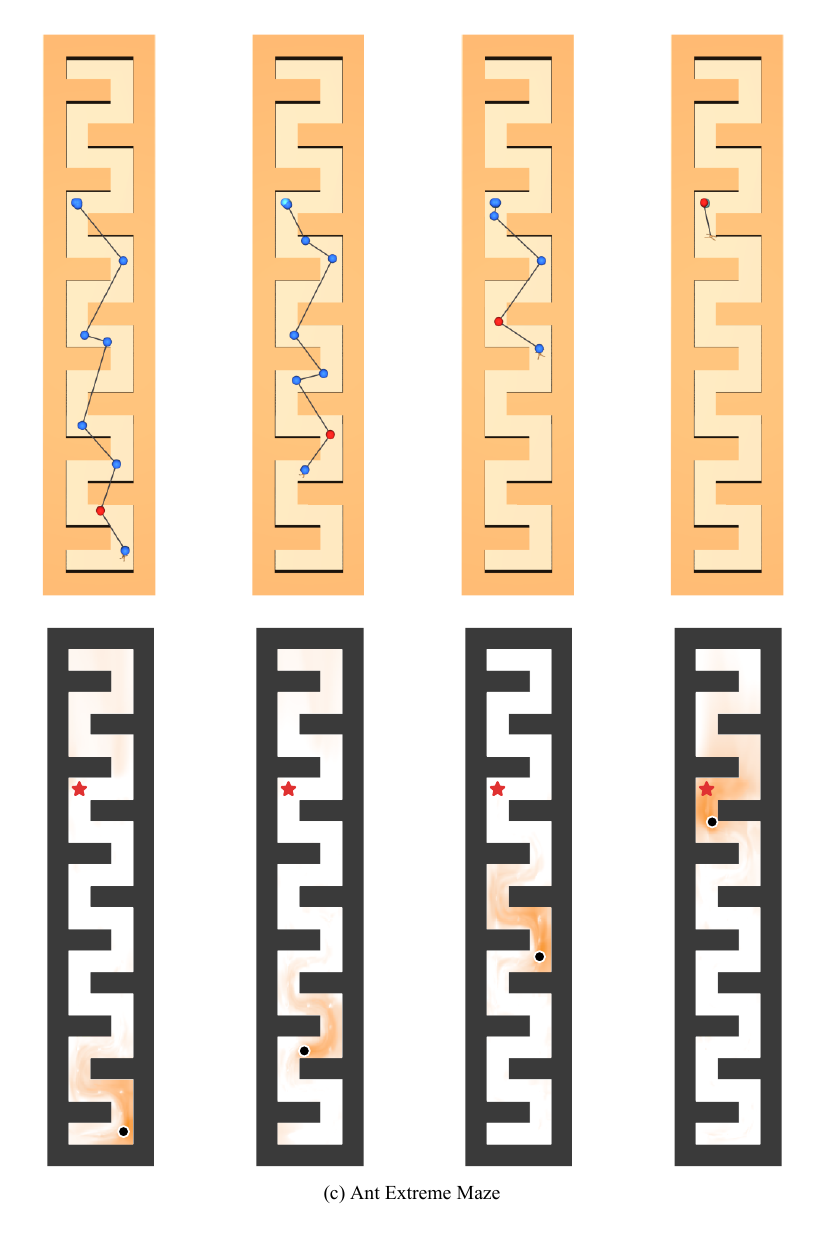}
    \vspace{-30pt}
    \caption{\textbf{Visualization of waypoint representations in maze environments.} 
    We visualize the inferred latent waypoints $z_{1:K}$ and the corresponding score heatmaps. 
    For each environment, the top row shows the $(x,y)$ positions corresponding to the inferred latent waypoints, where \raisebox{-0.37ex}{\includegraphics[height=2.0ex]{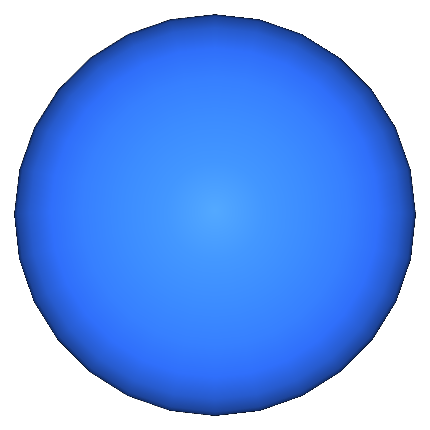}} denotes a candidate waypoint $z_i$, \raisebox{-0.37ex}{\includegraphics[height=2.0ex]{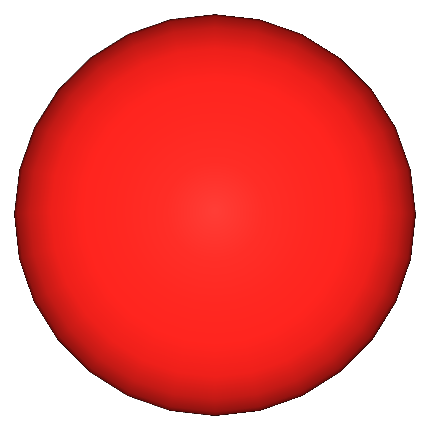}} highlights the selected waypoint $z_{i^*}$, and \raisebox{-0.37ex}{\includegraphics[height=2.0ex]{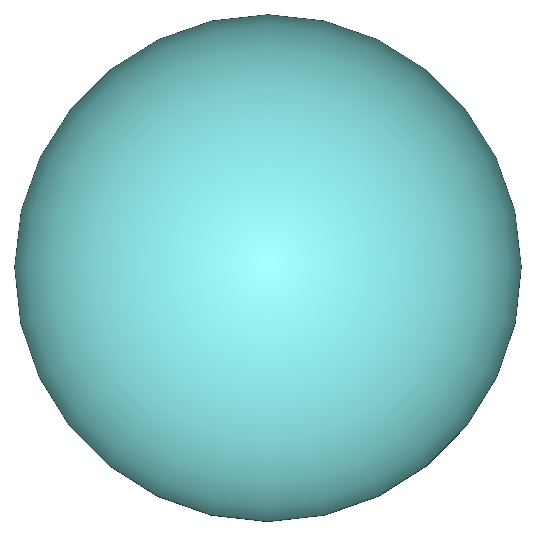}} represents the goal. 
    The bottom row shows heatmaps of $f_L(s,a,s')=\phi(s,a)^\top\psi(s')$, with the current state marked by \raisebox{-0.37ex}{\includegraphics[height=2.0ex]{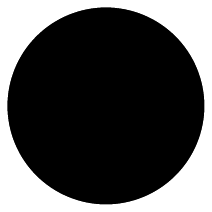}} and the goal indicated by \raisebox{-0.88ex}{\includegraphics[trim={18pt 18pt 18pt 18pt}, clip, height=3ex]{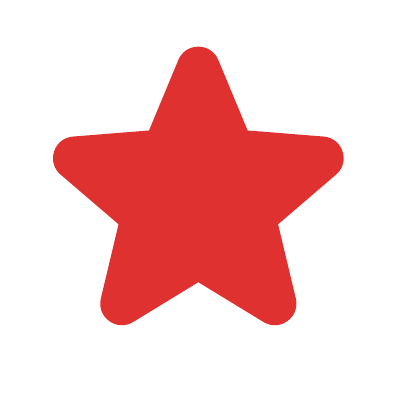}}.
    }
    \vspace{-5pt}
    \label{fig:waypoint_energy}
\end{figure}

\clearpage
\section{Experimental Details}
\label{sec:experimental_details}
We implement C-Slerp on top of the publicly available JaxGCRL~\citep{bortkiewicz2025accelerating} and BuilderBench~\citep{ghugare2025builderbench} codebases. 
For baselines, we implement CRL~\citep{eysenbach2022contrastive}, RIS~\citep{chane-sane2021ris}, an online version of HIQL~\citep{park2023hiql}, L3P~\citep{zhang2021l3p}, SAC+HER~\citep{haarnoja2018sac,andrychowicz2017hindsight} on the same codebase as C-Slerp.
We run our experiments on a local machine consisting of A6000 GPUs, and each run in Section~\ref{sec:experiments} takes 12-72 hours depending on the environment.
All training metric curves are averaged over 5 seeds, and the shaded region denotes one standard error.
We provide the environment and evaluation details in Appendix~\ref{sec:environments}, and the implementation details in Appendix~\ref{sec:implementation_details}.

\subsection{Environments}
\label{sec:environments}

\begin{figure}[t]
    \centering
    \includegraphics[width=1.0\linewidth]{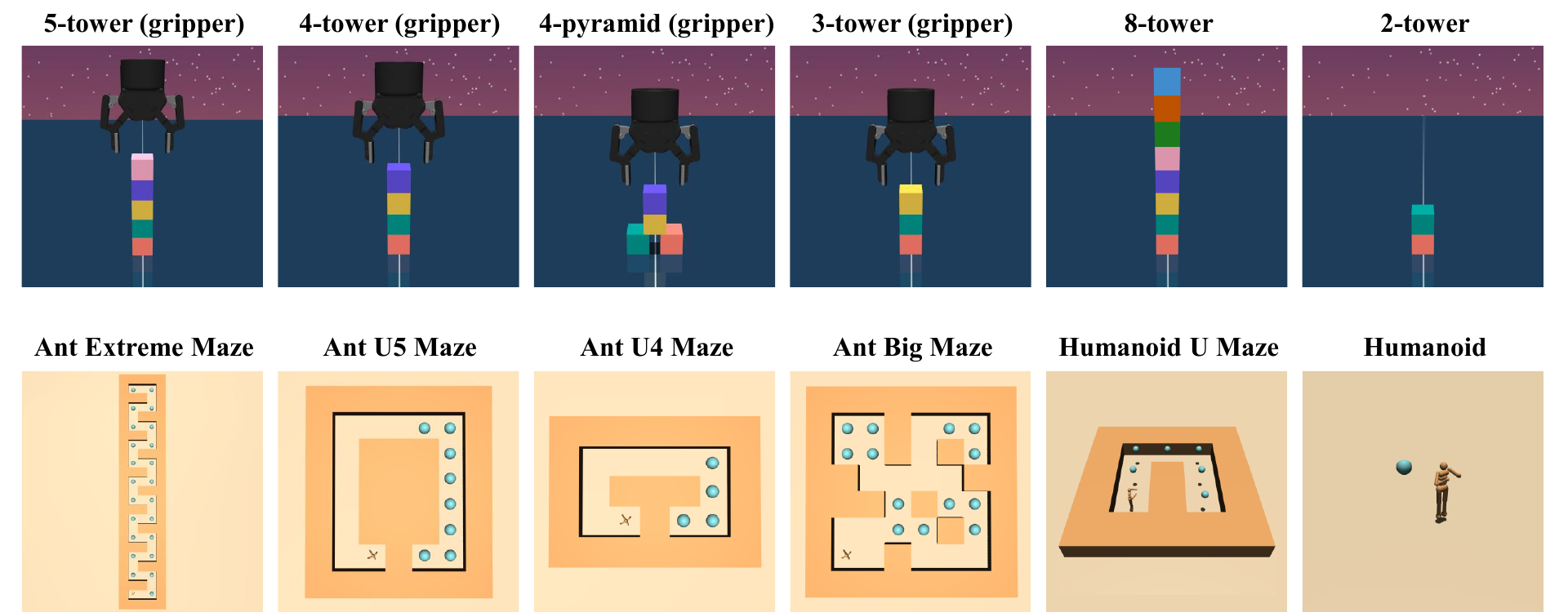}
    \vspace{-15pt}
    \caption{\textbf{Visualization of benchmark environments.}
    The top row shows the BuilderBench manipulation tasks, and the bottom row shows the JaxGCRL navigation tasks.
    The goal used for evaluation in each JaxGCRL environment is indicated by \raisebox{-0.37ex}{\includegraphics[height=2.0ex]{figures/appendixes/caption_markers_fig15/goal_teal_white.png}}. We provide videos demonstrating C-Slerp on all benchmark tasks on our \href{https://jonghaepark.github.io/c-slerp}{\textcolor{magenta}{project page}}.}
    \label{fig:env_explanation}
\end{figure}

\paragraph{JaxGCRL.}
We evaluate our method on goal-conditioned navigation tasks in JaxGCRL, including Humanoid, Humanoid U Maze, Ant Big Maze, Ant U4 Maze, and Ant U5 Maze. 
To extend our evaluation to substantially longer-horizon problems, we additionally design Ant Extreme Maze, inspired by \citet{myers2025horizon}.

The states consist of the torso’s global position and proprioceptive observations, including joint positions and velocities. The goals specify the desired torso position, and the actions correspond to joint-control commands.
For Ant and Humanoid, respectively, the states have 29 and 268 dimensions, the goals indicate the desired global torso coordinates $(x,y)$ and $(x,y,z)$, and the actions have 8 and 17 dimensions.
For the maze tasks, training goals are sampled uniformly from valid goal locations, whereas evaluation uses the goals shown as markers in Figure~\ref{fig:env_explanation}.
The maximum episode length is 1,000 environment steps, except in Ant Extreme Maze, where the limit is 2,000 steps.

\paragraph{BuilderBench.}
We evaluate our method on gripper-based manipulation tasks with three to five cubes in BuilderBench.
Each environment consists of $N$ cubes and a PD-controlled gripper. 
The gripper state consists of its 3-dimensional position, 4-dimensional orientation, 3-dimensional linear velocity, and 2 finger positions, giving 12 dimensions. 
Each cube contributes 13 dimensions consisting of its 3-dimensional position, 4-dimensional orientation, 3-dimensional linear velocity, and 3-dimensional angular velocity, yielding a $(12+13N)$-dimensional full state. 
Goals consist of the desired $(x,y,z)$ position of each cube and are $3N$-dimensional. 
The action space has 5 dimensions, controlling gripper opening, Cartesian motion, and yaw.
At each reset, the initial cube positions are randomized, and a randomly sampled $(x,y)$ offset is added to the desired position of the goal structure. 
To accommodate the longer manipulation sequences required as the number of cubes $N$ increases, we use $100+100N$ environment steps per episode during training and $200+200N$ environment steps per episode during evaluation.

Following \citet{yanscaling}, we additionally evaluate a modified BuilderBench environment in which the policy selects a cube and applies force and torque \textit{directly} to it without using a gripper. 
Each cube has the same 13 pose and velocity components as in the gripper environment, and an additional scalar indicates the selected cube, yielding a $(1+13N)$-dimensional state. 
Goals specify the desired $(x,y,z)$ position of each cube and are $3N$-dimensional. 
Each 5-dimensional action is represented as $(f_x,f_y,f_z,T_z,i)$, where $(f_x,f_y,f_z)$ denotes the applied force, $T_z$ denotes the yaw torque, and $i$ identifies the cube receiving the actuation. 
As in the gripper environment, the initial cube positions are randomized, and the same randomly sampled $(x,y)$ offset is added to the desired position of the goal structure. 
To account for the longer manipulation sequences required as the number of cubes increases, we set the maximum episode length to $100+50N$ environment steps during both training and evaluation.

\paragraph{Metrics.}
For both JaxGCRL and BuilderBench, we report the fraction of evaluation episodes in which the agent meets the success criterion for each task.
In JaxGCRL environments, success is achieved when the torso is within $0.5$\,m of the goal position. 
In BuilderBench environments, success requires all cubes to be within $0.01$\,m of their target heights and the structure to lie within $0.04$\,m of its target $(x,y)$ position.
Each cube has an edge length of $0.04$\,m. 

For the four-cube pyramid task, we also require the \textit{horizontal} arrangement of the cubes to match the target structure.
We measure each cube's horizontal position relative to the center of the bottom layer and compare it with the corresponding position in the target structure.
The allowed error is approximately $0.01$\,m for bottom-layer cubes and $0.03$\,m for upper-layer cubes.
The center of the bottom layer must also be within $0.04$\,m of its target position.

\subsection{Implementation details}
\label{sec:implementation_details}
During the learning process, we adjust the entropy coefficients $\alpha$ to match the target entropies using the automatic tuning procedure of SAC~\citep{haarnoja2018sac}. 
The number of waypoints $K$ is sampled from the softmax distribution in \Cref{eq:k_selection_distribution} with a temperature of $\eta=0.1$ during both data collection and evaluation.
Separately, to control the sharpness of the softmax distributions in the InfoNCE loss, $f_H$ and $f_L$ are divided by a temperature of $\tau_{\mathrm{NCE}}=0.1$ during training.
We also adopt the sampling strategy of \citet{ziarko2026contrastive}, which introduces in-trajectory negatives by sampling multiple positive pairs from each trajectory, with the number of pairs controlled by a repetition factor.
For BuilderBench, we set the repetition factor to 12.
For JaxGCRL, we use the default sampling strategy implemented in the codebase.
We provide the hyperparameters of C-Slerp in Table~\ref{tab:aplanning_hparams} and the source code on our \href{https://jonghaepark.github.io/c-slerp}{\textcolor{magenta}{project page}}.

Algorithm~\ref{alg:aplanning} provides the full pseudocode for C-Slerp.
Every $h$ environment steps, the high-level planner determines the number of waypoints and selects a subgoal for the low-level policy to pursue for the next $h$ steps.
The resulting transitions are stored in the replay buffer and used to update the representations and policy.

\paragraph{Baseline methods.}
We implement all baselines within the same training and evaluation framework as our method.
We closely follow each method's original training objectives and algorithm-specific components while using shared infrastructure and matching the capacity of the main networks where appropriate.

For CRL~\citep{eysenbach2022contrastive}, we closely follow the implementation provided in each benchmark's codebase, retaining its network architecture and L2-based energy function, while tuning the learning rates.

For RIS~\citep{chane-sane2021ris}, we follow the original diagonal Laplace subgoal
parameterization and goal-relabeling scheme. We use the rollout goal, an achieved goal sampled
independently from the replay buffer, and a future achieved goal from the same trajectory with
probabilities $0.2$, $0.4$, and $0.4$, respectively. Future achieved goals are sampled uniformly
from later states within the same trajectory. We draw 10 samples from the imagined-subgoal prior
and use the resulting subgoals only during training. At evaluation time, the policy is conditioned
directly on the final task goal.

For HIQL~\citep{park2023hiql}, we adapt the original offline RL algorithm to our online setting by
replacing the fixed offline dataset and pretraining stage with trajectories collected through
online interaction and updates from the resulting replay buffer. We preserve the
action-free expectile value-learning objective, the high- and low-level advantage-weighted policy
objectives, and the latent subgoal prediction used by the hierarchical policy.
For value learning, we use the current state, a future achieved goal from the same
trajectory, and an achieved goal sampled independently from the replay buffer with probabilities
$0.2$, $0.5$, and $0.3$, respectively. Future achieved goals are sampled from later states within
the same trajectory according to a geometric distribution with parameter $1-\gamma$.
For the high-level objective, we sample an achieved goal either from the future of the same trajectory or independently from the replay buffer with probabilities $0.7$ and $0.3$, respectively.
The target horizon $K$ determines how subgoals are selected during training.
The low-level policy is trained to reach a state selected from the next $K$ steps of the same trajectory, while the high-level policy is trained to generate subgoals using targets selected within the same horizon.
We set $K=25$ in JaxGCRL and $K=10$ in BuilderBench.

For SAC+HER~\citep{haarnoja2018sac,andrychowicz2017hindsight}, goal relabeling uses a future
achieved goal from the same trajectory with probability $0.8$ and the original rollout goal with
probability $0.2$. Future achieved goals are sampled uniformly from later states within the same
trajectory, and rewards are recomputed with respect to the resulting goal. Goal relabeling does
not alter the original environment termination or discount signal.

For L3P~\citep{zhang2021l3p}, the training-time planning graph contains 50 learned landmarks
and 150 additional replay-based landmarks, whereas evaluation uses only the 50 learned landmarks.
Landmark pairs are scored using directed reachability estimates, so the retained graph topology
depends on the learned values and the distance-clipping threshold rather than on a fixed edge set.
L3P's hindsight relabeling uses a uniformly sampled future achieved goal from the same trajectory
with probability $0.85$, with the future state restricted to at most 100 steps ahead. The graph is
refreshed every $5\times10^5$ environment interactions, and planning during data collection is
enabled after $5\times10^6$ interactions. To improve robustness in the online setting, each learned
latent component is represented by a state selected from the replay buffer rather than directly by
a decoded latent centroid. This ensures that planning landmarks correspond to states observed in
the environment and avoids centroids that may fall in invalid or unreachable regions.

For model architecture, CRL, RIS, SAC+HER, and L3P use the same backbone as our method for their
main policy and value-learning networks. For HIQL, we preserve its component-specific structure.
The representation and value networks use three hidden layers, the high-level policy uses two hidden layers, and the low-level policy matches the number of hidden layers in the policy used by C-Slerp for each benchmark.
To improve performance, both HIQL policies use half as many units per hidden layer as our method's policy. HIQL uses a 10-dimensional latent subgoal representation,
and L3P uses a 16-dimensional latent space. Unless otherwise specified, the baselines follow the
same shared training settings as our method.

For baseline tuning, we use settings reported in the original papers and implementations as
starting points and vary learning rates and method-specific parameters to identify the
best-performing configuration for each baseline and benchmark.
Table~\ref{tab:baseline_tuning} summarizes the candidate values considered during tuning.
Once selected, each configuration is fixed for all final evaluations.
Tables~\ref{tab:crl_hparams}--\ref{tab:l3p_hparams} report the selected method-specific
parameters used in our experiments.

\clearpage

\begin{algorithm}[t]
\caption{Contrastive Planning via Spherical Linear Interpolation (C-Slerp)}
\label{alg:aplanning}
\begin{algorithmic}[1]
\State Initialize policy $\pi_\theta(a\mid s,z)$,
encoders $\phi(s,a)$ and $\psi(s)$,
rotation matrix $R$, replay buffer $\mathcal{D}$
\vspace{0.1em}
\For{each training epoch}
    \vspace{0.5em}
    \item[] \hspace*{1.2em} {\color{SphereBlue}\texttt{// Data collection}}
    \vspace{0.5em}
    \For{each episode}
        \State $t \gets 0$
        \vspace{0.1em}
        \State Sample an initial state $s_0 \sim p_0$ and a goal $g \sim p_g$
        \While{the episode is not finished}
            \vspace{0.1em}
            \If{$t \mathbin{\%} h = 0$}
    \vspace{0.5em}
    \item[] \hspace*{5.6em} {\color{SphereBlue}\texttt{// K-selection}}
    \vspace{0.5em}
                \State Sample $K$ for $(s_t,g)$ using \Cref{eq:k_selection_distribution}
                \vspace{0.1em}
                \For{$i=1,\ldots,K$}
                    \State Compute $(\mu_i^\star,\sigma_i^{\star 2})$
                    using \Cref{eq:closed_form_high_policy}
                    \vspace{0.1em}
                    \State Sample $\tilde z_i \sim
                    \mathcal{N}(\mu_i^\star,\sigma_i^{\star 2}I)$
                    \vspace{0.1em}                
                    \State Set $z_i=\tilde z_i/\|\tilde z_i\|_2$
                \EndFor
                \vspace{0.1em}
                \State Set $z_{K+1}=\psi(g)$
    \vspace{0.5em}
    \item[] \hspace*{5.6em} {\color{SphereBlue}\texttt{// i-selection}}
    \vspace{0.5em}
                \For{$i=1,\ldots,K+1$}
                    \State Sample $\tilde a_i \sim
                    \pi_\theta(\cdot\mid s_t,z_i)$
                    \vspace{0.1em} 
                    \State Compute $\rho_i=\phi(s_t,\tilde a_i)^\top z_i$
                \EndFor
                \vspace{0.1em}
                \If{$\max_{i\in\{1,\ldots,K+1\}}\rho_i \geq \epsilon$}
                    \vspace{0.1em} 
                    \State Select $i^\star$ using \Cref{eq:waypoint_index_selection}
                \Else
                    \State Set $i^\star=1$
                \EndIf
                \vspace{0.1em} 
                \State Set the subgoal $z=z_{i^\star}$
                \vspace{0.1em}
            \EndIf
    \vspace{0.5em}
    \item[] \hspace*{4.1em}  {\color{SphereBlue}\texttt{// Policy rollout}}
    \vspace{0.5em}
            \State Sample $a_t \sim \pi_\theta(\cdot\mid s_t,z)$
            \vspace{0.1em} 
            \State Execute $a_t$ and observe $s_{t+1}$
            \vspace{0.1em} 
            \State Store $(s_t,a_t,s_{t+1})$ in $\mathcal{D}$
            \vspace{0.1em} 
            \State $t \gets t+1$
        \EndWhile
    \EndFor
    \vspace{0.5em}
    \item[] \hspace*{1.2em} {\color{SphereBlue}\texttt{// Representation and policy learning}}
    \vspace{0.5em}
    \For{each gradient step}
        \State Sample a batch
        $\mathcal{B}=\{(s_j,a_j,s_j^+)\}_{j=1}^{B}$ from $\mathcal{D}$
        \vspace{0.1em} 
        \State Update $\phi,\psi,R$ by minimizing
        $\mathcal{L}_{\mathrm{repr}}$ in \Cref{eq:repr_loss}
        \vspace{0.1em} 
        \State Update $\pi_\theta$ using \Cref{eq:policy_objective}
    \EndFor
\EndFor
\end{algorithmic}
\end{algorithm}

\clearpage

\begin{table}[t]
\centering
\small
\caption{Learning hyperparameters.}
\label{tab:aplanning_hparams}
\vspace{6pt}
\begin{tabular}{lcc}
\hline
Parameter & JaxGCRL & BuilderBench \\
\hline
Actor learning rate & $10^{-4}$ & $10^{-4}$ \\
Representation learning rate & $10^{-4}$ & $10^{-4}$ \\
Actor entropy learning rate & $3{\times}10^{-4}$ & $3{\times}10^{-4}$ \\
Waypoint entropy learning rate & $3{\times}10^{-3}$ & $3{\times}10^{-3}$ \\
Batch size & $512$ & $1024$ \\
Representation dimension $d$ & $128$ & $512$ \\
Network depth & $8$ & $16$ \\
Network width & $512$ & $1024$ \\
Discount factor $\gamma$ & $0.99$ & $0.99$ \\
InfoNCE temperature $\tau_{\mathrm{NCE}}$ & $0.1$ & $0.1$ \\
$\lambda_{\mathrm{repr}}$ & $4.0$ & $2.0$ \\
Control horizon $h$ & $10$ & $10$ \\
Maximum planning stages $K_{\max}+1$ & $5$ ($10$ for Ant Extreme Maze) & $4$ \\
$K$ sampling temperature $\eta$ & $0.1$ & $0.1$ \\
$i$-selection threshold $\epsilon$ & $0.4$ & $0.4$ \\
In-trajectory repetition factor & - & $12$ \\
Policy target entropy& $-|\mathcal{A}|$ & \begin{tabular}[c]{@{}c@{}}
    $-2|\mathcal{A}|$ \\
    ($-|\mathcal{A}|$ with gripper)
  \end{tabular} \\
Waypoint target entropy
& $-2.5(d-1)$
& \begin{tabular}[c]{@{}c@{}}
    $-3.5(d-1)$ \\
    ($-3.0(d-1)$ with gripper)
  \end{tabular} \\
\hline
\end{tabular}
\end{table}

\begin{table*}[t]
\centering
\small
\caption{Hyperparameter search spaces for baseline tuning.}
\label{tab:baseline_tuning}
\vspace{6pt}
\begin{tabular}{lll}
\hline
Method & Parameter & Search space \\
\hline
CRL
& Actor learning rate & $\{3\times10^{-4}, 2\times10^{-4}, 10^{-4}\}$ \\
& Critic learning rate & $\{3\times10^{-4}, 2\times10^{-4}, 10^{-4}\}$ \\
\hline
RIS
& Actor learning rate & $\{3\times10^{-5}, 10^{-4}, 3\times10^{-4}\}$ \\
& Critic learning rate & $\{10^{-4}, 3\times10^{-4}, 10^{-3}\}$ \\
& KL coefficient & $\{0.1, 0.3\}$ \\
& Prior samples  & $\{3, 10\}$ \\
\hline
SAC+HER
& Actor/critic learning rate & $\{10^{-4}, 3\times10^{-4}\}$ \\
& Target entropy & $\{-0.5|\mathcal{A}|, -1|\mathcal{A}|, -2|\mathcal{A}|\}$ \\
\hline
Online HIQL
& Learning rate  & $\{10^{-4}, 2\times10^{-4},  3\times10^{-4}\}$ \\
& High-level coefficient $\beta_{\mathrm{high}}$ & $\{0.1, 0.2, 0.3, 0.5, 1.0\}$ \\
& Low-level coefficient $\beta_{\mathrm{low}}$ & $\{0.3, 0.4, 0.5, 1, 100\}$ \\
& Target horizon $K$ & $\{3, 5, 10, 25\}$ \\
& Direct-goal bootstrap fraction & $\{0, 0.025\}$ \\
& Random-action warm-up & $\{0, 1000, 2000\}$ \\
\hline
L3P
& Actor/critic/value learning rate & $\{10^{-4}, 2\times10^{-4}, 3\times10^{-4}\}$ \\
& Discount factor $\gamma$ & $\{0.98, 0.99\}$ \\
& Action regularization & $\{0.05, 0.1\}$ \\
& Learned landmarks & $\{50, 100\}$ \\
& Distance-clipping threshold & $\{-20, -15\}$ \\
& Graph refresh interval & $\{10^5, 5\times10^5\}$ \\
& Exploration noise & $\{0.1, 0.2\}$ \\
& Random-action probability & $\{0.05, 0.2\}$ \\
\hline
\end{tabular}
\end{table*}

\clearpage

\begin{table}[t]
\centering
\small
\caption{CRL-specific parameters.}
\label{tab:crl_hparams}
\vspace{6pt}
\begin{tabular}{lcc}
\hline
Parameter & JaxGCRL & BuilderBench \\
\hline
Actor learning rate & $3\times10^{-4}$ & $3\times10^{-4}$ \\
Critic learning rate & $3\times10^{-4}$ & $3\times10^{-4}$ \\
Energy function & L2 norm & L2 norm \\

\hline
\end{tabular}
\end{table}

\begin{table}[t]
\centering
\small
\caption{RIS-specific parameters.}
\label{tab:ris_hparams}
\vspace{6pt}
\begin{tabular}{lcc}
\hline
Parameter & JaxGCRL & BuilderBench \\
\hline
Actor/subgoal learning rate & $10^{-4}$ & $10^{-4}$ \\
Critic learning rate & $10^{-4}$ & $10^{-3}$ \\
KL coefficient & 0.1 & 0.1 \\
Bottleneck temperature & 0.1 & 0.1 \\
Prior samples & 10 & 10 \\
\hline
\end{tabular}
\end{table}

\begin{table}[t]
\centering
\small
\caption{SAC+HER-specific parameters.}
\label{tab:sac_her_hparams}
\vspace{6pt}
\begin{tabular}{lcc}
\hline
Parameter & JaxGCRL & BuilderBench \\
\hline
Actor/critic learning rate & $10^{-4}$ & $10^{-4}$ \\
Entropy learning rate & $3\times10^{-4}$ & $3\times10^{-4}$ \\
HER probability & 0.8 & 0.8 \\
Initial entropy coefficient & 1.0 & 0.01 \\
Target entropy coefficient & $-0.5$ & $-1.0$ \\
\hline
\end{tabular}
\end{table}

\begin{table}[t]
\centering
\small
\caption{Online HIQL-specific parameters.}
\label{tab:hiql_hparams}
\vspace{6pt}
\begin{tabular}{lcc}
\hline
Parameter & JaxGCRL & BuilderBench \\
\hline
Learning rate & $3\times10^{-4}$ & $10^{-4}$ \\
Expectile & 0.7 & 0.7 \\
$\beta_{\mathrm{high}}$ & 0.2 & 0.5 \\
$\beta_{\mathrm{low}}$ & 0.3 & 100 \\
Target horizon $K$ & 25 & 10 \\
Direct-goal bootstrap fraction & 0 & 0 \\
Random-action warm-up & 0 & 0 \\
Representation dimension & 10 & 10 \\
\hline
\end{tabular}
\end{table}

\begin{table}[t]
\centering
\small
\caption{L3P-specific parameters.}
\label{tab:l3p_hparams}
\vspace{6pt}
\begin{tabular}{lcc}
\hline
Parameter & JaxGCRL & BuilderBench \\
\hline
Actor/critic/value learning rate & $10^{-4}$ & $10^{-4}$ \\
Autoencoder/cluster learning rate & $3\times10^{-4}$ & $3\times10^{-4}$ \\
Target-network update period & 10 learner steps & 10 learner steps \\
Latent dimension & 16 & 16 \\
Action regularization & 0.05 & 0.05 \\
Learned landmarks & 50 & 50 \\
Additional training landmarks & 150 & 150 \\
Graph refresh interval & $5\times10^5$ & $5\times10^5$ \\
Planning activation & $5\times10^6$ & $5\times10^6$ \\
Distance-clipping threshold & $-20$ & $-20$ \\
Soft-planning temperature & 1.1 & 1.1 \\
Value-iteration steps & 20 & 20 \\
Local planning horizon & 10 & 10 \\
Exploration noise & 0.2 & 0.2\\
Random-action probability & 0.2 & 0.2\\
\hline
\end{tabular}
\end{table}

\end{document}